\documentclass[twocolumn]{article}

\usepackage{graphicx}

\usepackage{hyperref}

\usepackage{amsmath}
\usepackage{amsthm}
\usepackage{amsfonts}
\usepackage{amssymb}
\usepackage{nccmath}
\usepackage{bbm}
\usepackage{mathtools}

\usepackage{comment}
\usepackage{caption}
\usepackage{subcaption}

\newtheorem{thm}{Theorem}
\newtheorem{cor}[thm]{Corollary}
\newtheorem{lem}[thm]{Lemma}

\theoremstyle{definition}
\newtheorem{rem}[thm]{Remark}

\newcommand{\RR}{ \mathbb{R} }

\newcommand{\spaceo}{\hspace{2 mm}}
\newcommand{\half}{\frac{1}{2}}

\newcommand{\Prob}[1]{\mathbb{P}\left( #1 \right)}

\newcommand{\Probsubidx}[2]{\mathbb{P}_{#1}\left( #2 \right)}

\newcommand{\Abs}[1]{\left| #1 \right|}
\newcommand{\Set}[1]{\left\{ #1 \right\}}
\newcommand{\Brack}[1]{\left( #1 \right)}

\newcommand{\inner}[2]{\left< #1 , #2 \right>}

\newcommand{\Expsubidx}[2]{ \mathbb{E}_{#1} #2}

\newcommand{\norm}[1]{\left\|#1\right\|}

\newcommand{\Ind}[1]{ \mathbbm{1}_{\Set{#1}} }
\newcommand{\eps}{\varepsilon}

\newlength{\dhatheight}

\newcommand{\mcF}{\mathcal{F}}
\newcommand{\mc}[1]{\mathcal{#1}}

\newcommand{\setsep}{ \spaceo \vert \spaceo}

\newcommand{\eRad}[1]{\mathfrak{R}\left(#1\right)}
\newcommand{\eRadidx}[2]{\mathfrak{R}_{#1}\left(#2\right)}
\newcommand{\eRadn}[1]{\eRadidx{n}{#1}}

\newif\ifimagesshow
\imagesshowfalse

\title{Dimension Free Generalization Bounds for Non Linear Metric Learning}

\author{
Mark Kozdoba$^1$ \\
{\tt\small markk@technion.ac.il}
\and
Shie Mannor$^1$ \\
{\tt\small shie@ee.technion.ac.il}
}
\date{
     $^1${\tt\small Technion, Israel Institute of Technology}
}

\begin{document}
\maketitle

\begin{abstract}
In this work we study generalization guarantees for the metric learning problem, where the metric is induced by a neural network type embedding of the data.  Specifically, we provide uniform generalization bounds for two regimes -- the sparse regime, and a non-sparse regime which we term \emph{bounded amplification}.  The sparse regime bounds correspond to situations where $\ell_1$-type norms of the parameters are small. Similarly to the situation in classification, solutions satisfying such bounds can be obtained by an appropriate regularization of the problem.  On the other hand, unregularized SGD optimization of a metric learning loss typically does not produce sparse solutions.  We show that despite this lack of sparsity, by relying on a different, new property of the solutions, it is still possible to provide dimension free generalization guarantees. Consequently, these bounds can explain generalization in non sparse real experimental situations. We illustrate the studied phenomena on the MNIST and 20newsgroups datasets. 
\end{abstract}

\section{Introduction}
\label{sec:intro}
Metric Learning, \cite{bellet_etal_2015_metric_learning_book}, is the problem of finding a metric $\rho$ on the space of features, such that $\rho$ reflects some semantic properties of a given task. Generally, the input can be thought of as a set of labeled pairs $\Set{((x_i,x'_i),y_i)}_{i=1}^n$, where $x_i,x'_i \in \RR^d$ are the features, and $y_i$ is the label, indicating whether $x_i$ and $x'_i$ should be close in the metric or far apart.  For instance, in face identification, \cite{schroff2015facenet}, features $x_i$ and $x'_i$ corresponding to the same face should be close in $\rho$, while different faces should be far apart.

Note that the above metric learning formulation is fairly general and one can convert supervised clustering, or even standard classification problems into metric learning simply by setting $y_i=1$ if $x_i$ and $x'_i$ have the same original label and $y_i=0$ otherwise \cite{davis2007information,weinberger_kilian_2009_LMNN,cao2016generalization,khosla2020supervised,chicco2020siamese}.

The metric $\rho$ is typically assumed to be the Euclidean metric  taken after a linear or non-linear \emph{embedding} of the features.
That is, we consider a parametric family $\mc{F}^*$ of embeddings into a $k$-dimensional space, $f^* :\RR^d \rightarrow \RR^k$, and set
\footnote{Note that $\rho$ in (\ref{eq:rho_def}) is not strictly a metric. Nevertheless, this terminology is common.}
\begin{equation}
\label{eq:rho_def}
\rho(x,x') = \rho_{k,f^*}(x,x')= \frac{1}{k} \norm{f^*(x) - f^*(x')}_{2}^2. 
\end{equation} 

\begin{figure}
\centering
\subcaptionbox{
\label{fig:newsgroup_tsne_a}
Raw features, after normalization and 500 dim. PCA.}{
\includegraphics[width=\linewidth,height =3cm]{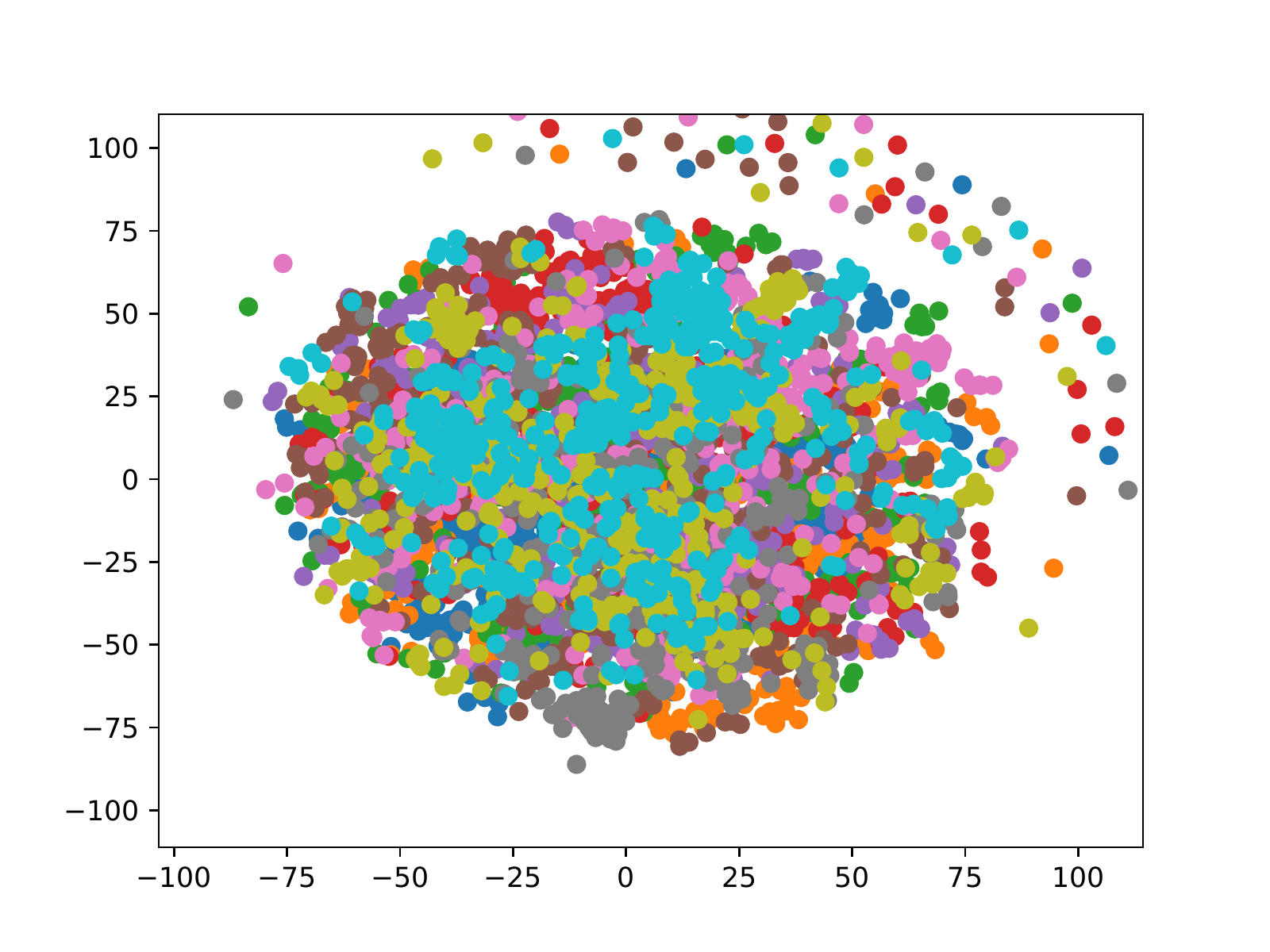}}
\\
\subcaptionbox{
\label{fig:newsgroup_tsne_b}
Learned embedding, k=50, test set.}{
\includegraphics[width=\linewidth,height =3cm]{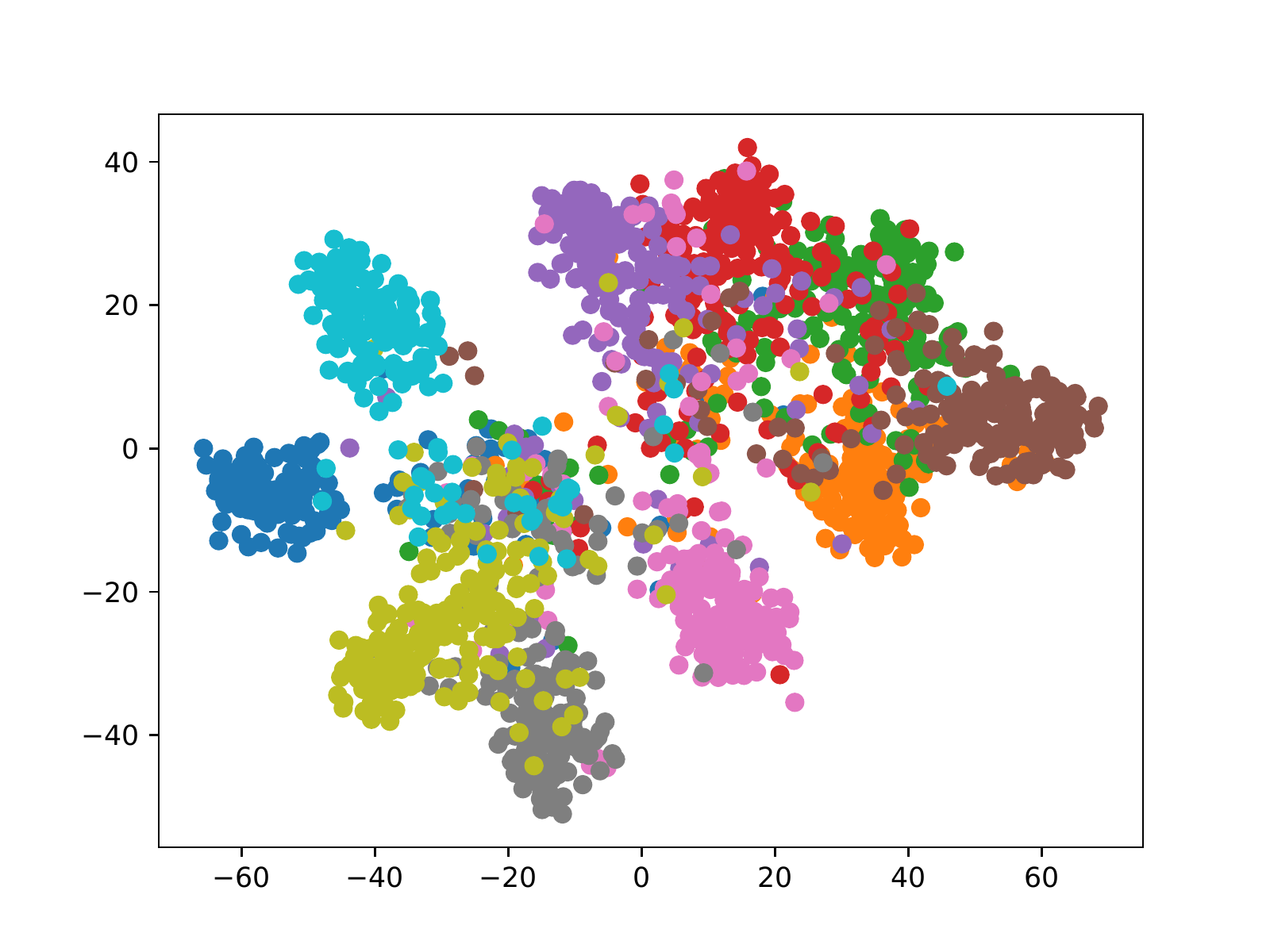}}
\caption{\label{fig:newsgroup_tsne}
tSNE plot of the 20newsgroup data, restricted to the first 10 labels. Points are colored according to the label. }
\end{figure}
As an example,  Figure \ref{fig:newsgroup_tsne} shows tSNE plots, \cite{tsne}, of the classical 20newsgroups dataset.  Figure \ref{fig:newsgroup_tsne_a} was generated using the regular bag of words representation of the data (see Section \ref{sec:experiments} for additional details) while Figure \ref{fig:newsgroup_tsne_b} was generated by applying tSNE to an embedding of the data (of the form (\ref{eq:non_lin_emebdding_one_layer_def}) below) learned by minimizing a loss based on labels, as described above. Clearly, there is no discernible relation between the label and the metric in the raw representation, but there is a strong relation in the learned metric. One can also obtain similar conclusions, and quantify them, by, for instance, replacing tSNE with spectral clustering.

The uniform \emph{generalization} problem for metric learning is the following: Given a family of embeddings $\mc{F}^*$, provide bounds, that hold uniformly for all $f^* \in \mc{F}^*$,  on the difference between the expected value of the loss and the average value of the loss on the train set; see Section \ref{sec:background_separate_section} for a formal definition.  Such bounds guarantee, for instance, that the train set would not be overfitted. Given a family of embeddings $\mc{F}^*$, consider the family $\mc{F}$ of functions $f: \RR^d \times \RR^d \rightarrow \RR$ of the form 
\begin{equation}
\label{eq:mcF_formal_def}
\mc{F} = \Set{ f(x,x') = \rho_{k,f^*}(x,x') \setsep f^* \in \mc{F}^*}.   \end{equation}
We refer to these as the distance functions, which map a pair of features into their distance after the embedding. Well known arguments imply that one can obtain generalization bounds for $\mc{F}^*$ by providing upper bounds on the \emph{Rademacher complexity} $\eRadn{\mc{F}}$ of the family of the scalar functions $\mc{F}$. Therefore in what follows, we discuss directly the upper bounds on $\eRadn{\mc{F}}$ for various families $\mc{F}^*$. We refer to Sections \ref{sec:background_separate_section} and  \ref{sec:notation} for formal definitions and details on the relation between $\mc{F}^*$,$\mc{F}$,$\eRadn{\mcF}$, and generalization bounds. 

\subsection{Sparsity Based Bounds}
\label{sec:intro_bounds}
In this paper our goal is to provide generalization bounds for neural network type non-linear embeddings $f^*$. Specifically, we will be interested in embeddings $f^*: \RR^d \rightarrow \RR^k$ of the form \begin{equation}
\label{eq:non_lin_emebdding_one_layer_def}
    f^*(x) = \phi(A^{t}x)
\end{equation}
where $A \in \RR^{d \times k}$ is a matrix and $\phi: \RR \rightarrow \RR$ is a Lipschitz non-linearity which acts on $A^{t}x$ coordinatewise. Equivalently, these embeddings are given by a single layer of a neural network. We note that the single layer case is the essential hard case for the metric learning problem. It is interesting in its own right, and moreover, multi-layer guarantees can easily be obtained by combining the single layer case with the methods used in the derivation of the existing neural network bounds. Additional details are provided in Remark \ref{rem:multi_layer_networks}. In view of this, in most parts of this paper we will concentrate on the single layer case.

Let $\norm{A}_{op}$ denote the spectral norm, and set $\norm{A}_{2,1} = \sum_{i=1}^k \norm{A_{\cdot i}}_2$, where $A_{\cdot i}$ is the $i$-th column of $A$. Let $\mc{F}^*$ be the set of embedding of the form (\ref{eq:non_lin_emebdding_one_layer_def}) induced by some set of matrices $\mc{G}$. In what follows we identify $\mc{G}$ and $\mc{F}^*$ and refer to $\mc{F}^*$ directly as the set of matrices, where the embeddings are assumed to be of the form $(\ref{eq:non_lin_emebdding_one_layer_def})$. We further assume that input feature norms are bounded by some $b>0$, $\norm{x}_2 \leq b$. Recall that $n$ is the number of samples.

Our first result is the following generalization bound (up to constants and logarithmic terms):
\begin{equation}
\label{eq:linear_bound_sparse}
    \eRadn{\mcF}\leq  \frac{ b^2 \norm{\phi}^2_{Lip} \sup_{A\in \mcF^*} \norm{A}_{op} \sup_{A\in \mcF^*} \norm{A}_{2,1} }{k \sqrt{n}}. 
\end{equation}
The full statement and proof of this result are given as Theorem \ref{thm:rad_direct_bound}.
This bound is the direct analog, in the metric learning setting, of the strongest currently known uniform bounds for neural network classification, \cite{bartlett2017spectrally}, which are also formulated in terms of $\norm{\cdot}_{op}$ and  $\norm{\cdot}_{2,1}$ norms. Bounds using these norms may generally be regarded as \emph{sparsity} type  bounds -- if we control the sum of the norms of coordinate functionals, then we control generalization. In particular, this bound will be small iff only a small number of output coordinate functionals (i.e. columns of $A$) have large norms.  This is similar to the standard $\ell_1$ regularized  logistic regression, where we want only a small number of coefficients to be large.

\subsection{Bounded Amplification Based Bounds}
If the $\norm{A}_{2,1}$ is forced to be small, for instance by adding an appropriate regularization term to the loss, then the learned $A$ will be sparse and (\ref{eq:linear_bound_sparse}) may explain generailzation. However, solutions obtained by stochastic gradient descent without regularization are not sparse. We illustrate this in Section \ref{sec:experiments}  using standard datasets.  In these experiments an embedding is learned by minimizing the loss with SGD without any regularization (or at least, without  $\norm{A}_{2,1}$ regularization), and we observe the behavior of the norms for a range of values of $k$. We find that on the one hand, even for large values of $k$ the generalization gap stays bounded, there is no overfitting. On the other hand, for the learned matrices $A$ the norm $\norm{A}_{2,1}$ grows linearly with $k$ while  $\norm{\cdot}_{op}$ grows as $\sqrt{k}$.  This in particular means 
that the left handside of (\ref{eq:linear_bound_sparse}) grows as $\sqrt{k}$, which suggests that it can not explain the observed generalization.

Consequently, to explain the above generalization, we introduce a different bound. Denote $\norm{A}_{2,\infty} = \max_{i\leq k} \norm{A_{\cdot i}}_2$. Then (again, up to logarithmic terms) we have:
\begin{equation}
\label{eq:linear_bound_non_sparse}
    \eRadn{\mcF}\leq \frac{ b^2 \sup_{A\in \mcF^*} \norm{A}_{2,\infty}^2 }{\sqrt{n}}.
\end{equation} 
This will be shown in Theorem \ref{thm:rad_non_sparse_dim_free}. In the experiments we observe that $\norm{A}_{2,\infty}$ remains roughly constant as a function of $k$,  and therefore this bound does explain the generalization.  Even more importantly, however, the bound (\ref{eq:linear_bound_non_sparse}) is remarkable because there is no sparsity control term for columns and yet it is still \emph{dimension free}. That is, the feature and embedding dimensions, $d$ and $k$, do not enter the bound.

Put differently, typically sparsity implies generalization because while the problem may have many parameters,  only few of them will be non-zero at each instance of the problem, and thus ``effective'' number of parameters is low, \cite{hastie2015statistical,bartlett2017spectrally,zhou2019NN_compression}. In contrast,  in (\ref{eq:linear_bound_non_sparse}), all $k$ columns can have the same non zero Euclidean norm $a$, the ``effective'' number of parameters can be arbitrarily high,
and yet the bound will not deteriorate when the number of samples $n$ is fixed and $k$ is growing. We refer to  situations where only $\norm{A}_{2,\infty}$ is bounded as the \emph{bounded amplification} regime, as distinguished from the sparse regime. The boundedness here refers to the fact that each column individually still should be $\ell_2$ bounded.

It is important to note that the normalization by $k$ in (\ref{eq:rho_def}) is crucial to the fact that (\ref{eq:linear_bound_non_sparse}) is dimension free. In Remark  \ref{rem:normalization_of_metric_by_k} we discus in detail why this is an appropriate normalization for this case. It is also worth noting that among the bounds (\ref{eq:linear_bound_sparse}) and (\ref{eq:linear_bound_non_sparse}) none is stronger than the other, i.e., it is not true that (\ref{eq:linear_bound_sparse}) implies 
(\ref{eq:linear_bound_non_sparse}) or vice versa. Rather, they apply to genuinely different types of embeddings $f^*$, those that minimize the $\norm{\cdot}_{2,1}$-regularized loss, and those that minimize the unregularized one.

\subsection{Methods}
\label{sec:intro_methods}
We now briefly describe the ideas underlying the proofs. To prove (\ref{eq:linear_bound_sparse}) we take the classical approach of bounding the covering numbers of the set of $\norm{\cdot}_{2,1}$ bounded matrices, using an estimate from \cite{bartlett2017spectrally}, which in turn is an extension of estimates in \cite{zhang2002covering}, \cite{bartlett_sample_1998}. 
The Rademacher bound then follows from the Dudley entropy bound. However, this method does not seem to yield (\ref{eq:linear_bound_non_sparse}), and provides bounds that are off by a $\sqrt{k}$ factor. Thus, to prove (\ref{eq:linear_bound_non_sparse}) we take a different approach.

We give two proofs for (\ref{eq:linear_bound_non_sparse}),  both of which exploit the particular structure of the metric (\ref{eq:rho_def}) as an average over the coordinates. Equivalently, the distance $\rho_{k,f}$ can be viewed as a voting outcome by an ensemble of individual coordinates. The first argument is based on the sub-additivity of Rademacher complexities and makes a direct use of the average structure as above. The second argument is a dimension reduction scheme. This  argument gives a somewhat weaker, although still dimension independent result, but has the advantage of making it much clearer \emph{why} there is dimension independent generalization. We also believe this argument has greater potential in generalizing to other situations. This dimension reduction was inspired by \cite{schapire_boosting_1998}, where a related approach was used to analyze ensemble learning.

\subsection{Contributions}
To conclude this Section we  summarize the contributions of this work: (\textbf{i}) 
For metric learning with non linear embeddings, we prove 
generalization bounds which depend only on the norms $\norm{\cdot}_{2,1}$, $\norm{\cdot}_{op}$ of the embedding matrix, similarly to the recent neural network classification bounds.  (\textbf{ii}) We introduce a new type of bounds, the bounded amplification type bounds,  with dependence only on the $\norm{\cdot}_{2,\infty}$ norm. Such bounds have not been studied before, in metric learning, or in statistical learning theory in general. The methods of proof here are also new. (\textbf{iii}) We observe empirically that the above bounded amplification regime of matrix weights occurs naturally in basic experiments on standard datasets. Our new bounds therefore can explain, in standard situations, dimension independence that can not be explained by the more common types of bounds.

The rest of this paper is organized as follows: 
In Section \ref{sec:background_separate_section} we overview the necessary background on metric learning and generalization.
Literature and related work are discussed in Section \ref{sec:literature}. In Section \ref{sec:notation} we provide the notation required for the formal discussion of the results.  
Section \ref{sec:results} contains the statements of the results and overviews of the proofs. Some of full proofs are deferred to the Supplementary Material due to space constraints. Experiments are described in Section \ref{sec:experiments}. In Section \ref{sec:conclusion} we conclude the paper and discuss future work.

\section{Metric Learning Background}
\label{sec:background_separate_section}
As discussed in Section \ref{sec:intro}, we assume that the training data is given as a set $\Set{((x_i,x'_i),y_i)}_{i=1}^n$ of labeled feature pairs, which we assume to be sampled independently form some distribution $\mc{D}$ on $\RR^d \times \RR^d \times Y$, where $Y$ is some set of label values. It is usually sufficeint to take $Y=\Set{0,1}$. Note that there may be dependence \emph{within} the pair, i.e. $x'_i$ may depend on $x_i$.

The quality of the embedding $f^*$ on a data point $((x,x'),y)$ is measured via a \emph{loss function} $\ell$, 
which usually depends on $((x,x'),y)$ only through the metric $\rho_{k,f^*}$. That is, we assume that there is a fixed function $\ell:\RR \times Y \rightarrow [0,1]$,  such that the loss of $f^* \in \mcF^*$ on a data point $((x,x'),y)$ is given by $L_{f^*}((x,x'),y) = \ell(\rho_{k,f^*}(x_i,x_i'),y_i)$.
A typical example of a loss $\ell$ is the following version of the margin loss:
\begin{equation*}
\ell_{S,D}^{\lambda}(\rho,y) = 
\begin{cases}
\min \Brack{1, \lambda \cdot ReLU\Brack{\rho - S}} & \text{ if } y=1 \\
\min \Brack{1, \lambda \cdot ReLU\Brack{D-\rho}} & \text{ otherwise}, 
\end{cases}
\end{equation*}
where $ReLU(x) = \max\Brack{0,x}$. As discussed in Section \ref{sec:intro}, this loss embodies the principle that $(x,x')$ should be close iff $y=1$, by penalizing distances above $S$ 
when $y=1$ and penalizing distances below $D$ otherwise.

The overall loss on the data, or the empirical risk, is given by 
\begin{equation}
\label{eq:empirical_risk_def}
\hat{R}_{f^*} = \frac{1}{n} \sum_{i=1}^n L_{f^*}((x_i,x_i'),y_i).
\end{equation}
The expected risk is given by 
$R_{f^*} = \Expsubidx{((x,x'),y) \sim \mc{D}}{L_{f^*} ((x,x'),y)}$. The uniform generalization problem of metric learning is similar to the generalization problem of classification: One is interested in conditions on the family $\mc{F}^*$ under which the gap between expected and empirical risks, $R_{f^*} - \hat{R}_{f^*}$, is small for all $f^* \in \mc{F}^*$.

Finally, since $\Set{((x_i,x'_i),y_i)}_{i=1}^n$ are independent, standard results imply that to control the uniform generalization bounds of the risk, it is sufficient to control the Rademacher complexity of the family $\mcF$ of distance value functions induced by $\mcF^*$, as defined in (\ref{eq:mcF_formal_def}). Specifically,  we have that (see \cite{mohri_book_2018} Theorem 3.3 and Lemma 5.7) ,  
\begin{equation*}
\sup_{f\in \mcF} R_{f^*} - \hat{R}_{f^*} \leq  2 \norm{\ell}_{Lip }\eRadn{\mcF}  + \sqrt{\frac{\log \delta^{-1}}{2n}} 
\end{equation*}
holds with probability at least $1-\delta$. Here
$\eRadn{\mcF}$ is the Rademacher complexity of $\mcF$ and $\norm{\ell}_{Lip}$ is the Lipschitz constant of $\ell$ as a function of its first coordinate.

\section{Literature}
\label{sec:literature}
A general survey of the field of metric learning can be found in \cite{bellet_etal_2015_metric_learning_book}. See also \cite{chicco2020siamese} for a survey of recent applications in deep learning contexts. In these situations, the metric learning loss is sometimes referred to as a Siamese network.

Up to now, generalization guarantees in metric learning were only studied in the \emph{linear} setting, i.e. for embeddings of the form (\ref{eq:non_lin_emebdding_one_layer_def}) where $\phi(x) = x$. In particular, all literature cited in this Section deals with the linear case.

As discussed in Sections \ref{sec:intro} and  \ref{sec:background_separate_section}, in this paper we use the formal setting introduced in \cite{verma_branson_etal_2015_sample_complexity_mahalanobis}, where we assume that the data comes as a set of iid feature pairs with a label per pair, $((x_i,x_i'),y_i)_{i=1}^n$. In this setting, the empirical risk is given by (\ref{eq:empirical_risk_def}). In the special case where the data comes with a label per feature, $(x_i,l_i)_{i=1}^n$, 
one can use an alternative notion of empirical risk, given by 
\begin{equation}
\label{eq:empirical_risk_quadratic}
\tilde{R}_f = \frac{1}{n(n-1)} \sum_{i=1}^n \sum_{j\neq i} L_f \Brack{(x_i,x_j),\Ind{l_i \neq l_j}},
\end{equation}
which is viewed as a second order U-statistic, see \cite{cao2016generalization}. That is, instead of considering the input as a set of independently sampled pairs, which can be obtained from the $(x_i,l_i)_{i=1}^n$ by, for instance, creating pairs out of consecutive samples, 
in (\ref{eq:empirical_risk_quadratic}) one considers all possible pairs. On one hand, compared to the risk $\hat{R}_f$ defined by (\ref{eq:empirical_risk_def}), for small datasets $\tilde{R}_f$ might make a somewhat better use of the data. On the other hand, $\tilde{R}_f$ is less general, since the data does not necessarily have to be generated by the label-per-feature setting, and moreover, even evaluating (\ref{eq:empirical_risk_quadratic}) may be computationally difficult since the number of terms in the sum is quadratic in dataset size. In practice, the empirical loss used is somewhere between $\hat{R}_f$ and $\tilde{R}_f$.

The generalization setting considered here is that of the \emph{uniform generalization bounds}, see  \cite{mohri_book_2018}. Alternatively, in \cite{wang2019multitask_metric_learning}, \cite{lei_etal_2020_sharper_pairwise_learning}, 
the linear case of metric learning was studied in the framework of algorithmic stability (\cite{be_stability}, \cite{feldman_vondrak}, \cite{bousquet2020sharper}).
In these works,  stability, and consequently generalization bounds, were obtained for an appropriately regularized empirical risk minimization (ERM) procedure. We note that these methods rely strongly on the (uniform) convexity of the underlying problem, and are unlikely to be generalized to non linear settings.  In \cite{bellet2015robustness},\cite{christmann2016robustness}, 
 linear metric learning was studied in the framework of algorithmic robustness, \cite{xu2012robustness}.

Finally, uniform bounds for the linear case were studied in  \cite{verma_branson_etal_2015_sample_complexity_mahalanobis}, 
and similar results for the version of risk as in (\ref{eq:empirical_risk_quadratic})  were obtained in \cite{cao2016generalization}. In particular it was shown in \cite{verma_branson_etal_2015_sample_complexity_mahalanobis}  that 
\begin{equation}
\label{eq:linear_bound}
    \eRadn{\mcF}\leq \frac{ b^2 \sup_{A\in \mcF^*} \norm{AA^t}_F}{k \sqrt{n}}.
\end{equation}
One can show that both (\ref{eq:linear_bound_sparse}) and (\ref{eq:linear_bound_non_sparse}) can be derived from this bound using known matrix norm inequalities. The bound (\ref{eq:linear_bound}) itself is derived in \cite{verma_branson_etal_2015_sample_complexity_mahalanobis} using a relatively short elegant argument involving only the Cauchy Schwartz inequality.  However, this argument can be applied only when $\phi$ is the identity. Similarly to the situation in classification, for the non-linear settings other arguments are required.

\section{Notation}
\label{sec:notation}

For a vector $v=(v_1,\ldots,v_m)\in \RR^m$, the $\ell_p$ norm is denoted by $\norm{v}_p = \Brack{\sum_{j=1}^m \Abs{v_j}^p}^{1/p}$. 
For a matrix $A \in R^{d\times k}$, and $1\leq p,s\leq \infty$, denote
$\norm{A}_{p,s} = \norm{\Brack{\norm{A_{\cdot 1}}_{p}, \ldots,\norm{A_{\cdot k}}_{p}}}_s$. That is, one first computes the $p$-th norm of the columns and then the $s$-th norm of the vector of these norms. Note that $\norm{A}_{2,2} = \norm{A}_2$ is the Frobenius norm. Denote by $\norm{A}_{op}$ the spectral norm of $A$. Throughout $\phi :\RR \rightarrow \RR$ will denote the non-linearity, with Lipschitz constant $\norm{\phi}_{Lip}$ and such that $\phi(0) = 0$.

The $\eps$ covering number of a set $\mc{A} \subset \RR^n$, denoted $\mc{N}(\mc{A},\eps)$ is the minimal size of a set $\mc{B} \subset \mc{A}$ such that for every $x\in \mc{A}$ there is $y\in \mc{B}$ s.t. $\norm{x-y}_2 \leq \eps$.
The Rademacher complexity of a set $\mcF \subset \RR^n$  is defined by 
\begin{equation}
    \eRad{\mcF} = \frac{1}{n}\Expsubidx{\eps}{\sup_{f\in \mcF} \sum_{i=1}^n \eps_i f_i}
\end{equation}
where $\eps_i$ are independent Bernoulli variables with $\Prob{\eps_i=1}=\Prob{\eps_i=-1}=\half$. If $\mcF$ is a family of functions rather than a subset of $\RR^n$, as in $(\ref{eq:mcF_formal_def})$, then we set $\eRadn{\mcF} := \eRad{\mcF_{| x,x'}}$ where $\mcF_{| x,x'} \subset \RR^n$ is the restriction of $\mcF$ to the training set (see also Section \ref{sec:sparse_proof}).

We use the standard notation 
$O(\cdot)$ and $\Omega(\cdot)$ for upper bounded and equivalent, respectively, up to absolute constants. $\overline{O}(\cdot)$ will denote upper boundedness up to absolute constants and logarithmic terms.

\section{Results}
\label{sec:results}
In Section \ref{sec:sparse_proof} we sate the sparse regime bound, Theorem \ref{thm:rad_direct_bound}, and provide an overview of the proof. The full proof is given in Supplementary Material Section \ref{sec:proof_of_thm_rad_sparse}. We also derive Corollary \ref{cor:rad_non_sparse}, which is a version of the bounded amplification bound, but with an additional $\sqrt{k}$ factor. Both proofs of the non sparse bound apply Corollary \ref{cor:rad_non_sparse} with smaller values of $k$.

In Section \ref{sec:non_sparse_proof} we state the bounded amplification result, Theorem \ref{thm:rad_non_sparse_dim_free}. As discussed in Section \ref{sec:intro_methods}, we give two proofs for Theorem \ref{thm:rad_non_sparse_dim_free}, the first of which is given in Supplementary Material Section \ref{sec:appendix_additivity_proof_of_non_sparse}. We then discuss the dimension reduction argument and prove the fundamental underlying approximation result, Lemma \ref{lem:approximation_f_k_fl}. The derivation of the Rademacher complexity bound from Lemma \ref{lem:approximation_f_k_fl} is given in Supplementary Material Section \ref{sec:appendix_dim_red_proof_of_non_sparse_result}.

\subsection{Sparse Case}
\label{sec:sparse_proof}
We use the setting described in Section \ref{sec:intro}. Assume that we are given feature pairs $\Set{(x_i,x'_i)}_{i=1}^n \in \RR^d \times \RR^d$, where each pair was sampled independently from some distribution on such pairs. Note that there may be  dependence inside pairs, i.e. $x_i$ may depend on $x'_i$, only the pairs themselves are assumed independent. These inputs are organized as two matrices, $X,X' \in \RR^{n\times d}$, with rows $x_i$ and $x'_i$ respectively. 

Define a shorthand for the normalized squared  Euclidean norm on $\RR^k$, $\zeta_k(x,x') = \frac{1}{k} \norm{x-x'}_2^2$ for 
 $x,x' \in \RR^k$.  Let $\mc{G} \subset \RR^{d\times k}$ be a set of matrices. We will be interested in the family of vectors
\begin{flalign*}
    \mcF_{k}(\mc{G}) &= \mcF_{X,X',k}(\mc{G}) \\
    &= \Set{  \Brack{\zeta_k(\phi(A^tx_i), \phi(A^tx'_i)
    }_{i=1}^n \setsep A\in \mc{G}} \subset \RR^n,
\end{flalign*}
for various sets $\mc{G}$ of matrices. In words, $\mcF_{k}(\mc{G})$ is the set of distance functions $f(x,x') = \zeta_k(\phi(A^tx), \phi(A^tx')$ induced by $A \in \mc{G}$, restricted to the input 
$\Set{(x_i,x'_i)}_{i=1}^n$. As discussed in Section \ref{sec:background_separate_section}, bounds on $\eRad{\mcF_{k}(\mc{G})}$ imply generalization bounds for metric learning with matrices in $\mc{G}$.

\begin{thm}
\label{thm:rad_direct_bound}
Given $a,a' >0$, set 
\begin{equation}
\label{eq:direct_bound_thm_g_def}
\mc{J}^k_{a,a'} = \Set{A \setsep A \in \RR^{d\times k}, \norm{A}_{2,1}\leq a, 
\norm{A}_{op}\leq a'}.
\end{equation}
Assume that $\norm{x_i}\leq b$ for all $i\leq n$. Then 
\begin{equation}
\label{eq:direct_bound_thm_g_result}
\eRad{\mcF_{k}(\mc{J}^k_{a,a'})} \leq 
 \overline{O} \Brack{\frac{1}{n} + 
 \frac{a a' b^2  \norm{\phi}^2_{Lip}}{k\sqrt{n}} }.
\end{equation}
\end{thm}
We now briefly sketch the proof of Theorem \ref{thm:rad_direct_bound}. The full proof is given in Supplementary Material Section \ref{sec:proof_of_thm_rad_sparse}. 
As discussed in Section \ref{sec:intro_methods}, to prove Theorem \ref{thm:rad_direct_bound}, we bound the covering numbers $\mc{N}(\mcF_{k}(\mc{J}^k_{a,a'}),\eps)$. The Rademacher complexity bound then follows using standard considerations, via the Dudley entropy integral bound. 

To bound $\mc{N}(\mcF_{k}(\mc{J}^k_{a,a'}),\eps)$ we will use following covering lemma from \cite{bartlett2017spectrally}, stated for the case of $\norm{A}_{2,1}$ norms.  
\begin{lem}[\cite{bartlett2017spectrally}] 
\label{lem:A_21_regularized_bounds}
For any $a>0$, set  
\begin{equation}
\mc{J}^k_{a} = \Set{A \setsep A \in \RR^{d\times k}, \norm{A}_{2,1}\leq a} \subset \RR^{d \times k}, 
\end{equation}
and for any $X \in \RR^{n\times d}$ and a set of matrices $\mc{G} \subset \RR^{d\times k}$ set
$X  \mc{G} =  \Set{XA \setsep A \in \mc{G}} \subset \RR^{n \times k}$. Then for any $\eps>0$,
\begin{equation}
\label{eq:lem:A_21_regularized_bounds_statement}
    \log \mc{N}(X\mc{J}^k_a,\eps) \leq \frac{a^2\norm{X}_2^2}{\eps^2} \log \Brack{2dk}.
\end{equation}
\end{lem}
The proof of this Lemma in \cite{bartlett2017spectrally} is based on Maurey's Lemma, similarly to the arguments in \cite{zhang2002covering},\cite{bartlett_sample_1998}. We note that one can instead directly estimate the supremum of the Rademacher process indexed by $X\mc{J}_a^k$. Then Sudakov minoration, \cite{vershynin_high_dimensional_2018}, would yield a slight improvement upon (\ref{eq:lem:A_21_regularized_bounds_statement}), with no $d$ dependence inside the logarithm.

Next, define a set of $n \times 2k$ matrices $\mc{Z}$ by $\mc{Z} = \Set{ (XA,X'A) \setsep A \in \mc{J}_{a,a'}^k } \subset \RR^{n\times 2k}$. Let $\phi \mc{Z} := \Set{\phi(z) \setsep z\in \mc{Z}}$  where $\phi$ is applied coordinatewise, and let $\hat{\zeta_k}:\RR^{d \times 2k} \rightarrow \RR^n$ be the mapping that applies $\zeta_k$ to the rows of elements in $\RR^{d \times 2k}$. Observe that by definition, $\mcF_{k}(\mc{J}^k_{a,a'}) = \hat{\zeta_k}\Brack{\phi \mc{Z}}$. Now, using Lemma \ref{lem:A_21_regularized_bounds}, we can bound $\mc{N}\Brack{\mc{Z},\eps}$. Then using Lipschitzity of the mappings $\mc{Z} \mapsto \phi \mc{Z}$ and $\hat{\zeta_k}$, this translates to a bound on $\mc{N}(\mcF_{k}(\mc{J}^k_{a,a'}),\eps)$.

We now consider the bounded amplification case Corollary: 
Given $a>0$, set 
\begin{equation}
\label{eq:direct_bound_non_sparse_G_def}
\mc{G}^k_a = \Set{A \setsep A \in \RR^{d\times k}, \norm{A}_{2,\infty}\leq a }.
\end{equation}

\begin{cor}
\label{cor:rad_non_sparse} 
Assume that $\norm{x_i}\leq b$ for all $i\leq n$. Then for the set $\mc{G}_a^k$ defined by (\ref{eq:direct_bound_non_sparse_G_def}) we have
\begin{equation}
\eRad{\mcF_{k}(\mc{G}^k_a)} \leq 
 \overline{O}\Brack{\frac{1}{n} + 
 \frac{a^2 b^2 \sqrt{k} \norm{\phi}^2_{Lip}}{\sqrt{n}} }.
\end{equation}
\end{cor}
The proof is given in Supplementary Material Section \ref{sec:cor_rad_non_sparse_proof}.

\subsection{Bounded Amplification Case}
\label{sec:non_sparse_proof}

\begin{thm}
\label{thm:rad_non_sparse_dim_free}
Assume that $\norm{x_i}\leq b$ for all $i\leq n$. Then for the set $\mc{G}_a^k$ defined by (\ref{eq:direct_bound_non_sparse_G_def}) we have
\begin{equation}
\eRad{\mcF_{k}(\mc{G}^k_a)} \leq 
 \overline{O}\Brack{1/n + \frac{a^2 b^2\norm{\phi}^2_{Lip}}{\sqrt{n}}}.
\end{equation}
\end{thm}
The proof of this result is given in  Section \ref{sec:appendix_additivity_proof_of_non_sparse} of the supplementary material. We will now discuss a dimension reduction proof, which yields a somewhat weaker, but still dimension independent bound:
\begin{equation}
\label{eq:non_sparse_bound_n_quater}
\eRad{\mcF_{k}(\mc{G}^k_a)} \leq 
 \overline{O}\Brack{1/n + \frac{a^2 b^2\norm{\phi}^2_{Lip}}{n^{1/4}}}.
\end{equation}
To make notation more compact, for any $l>0$ denote $\mc{F}_l = \mcF_{l}(\mc{G}^l_a)$. The main step of the proof 
of (\ref{eq:non_sparse_bound_n_quater}), Lemma \ref{lem:approximation_f_k_fl}, is to show that every $f \in \mc{F}_k$ can be approximated by $\hat{f} \in \mc{F}_{l}$ for any $l\leq k$, such that, roughly, $\norm{f-\hat{f}}_{\infty} \leq 1/\sqrt{l}$. This can be equivalently restated as follows: The set $\mcF_l$, of $l$-dimensional distance functions, approximates 
\emph{any} $\mcF_k$ up to $1/\sqrt{l}$, independently of $k$. We refer to this property as \emph{l-uniform approximability}, with rate $1/\sqrt{l}$. 
\begin{lem} 
\label{lem:approximation_f_k_fl}
For any $l\leq k$, for any $f \in \mc{F}_k$, there is $\hat{f} \in \mc{F}_l$ such that 
\begin{equation}
    \norm{f-\hat{f}}_{\infty} \leq 
    \frac{\Brack{a b\norm{\phi}_{Lip}}^2 \sqrt{\log 4n}}{\sqrt{l}}.
\end{equation}
\end{lem}
\begin{proof}
Fix $i\leq n$, and $A\in \mc{G}^k_{a}$. For $j\leq k$ set $r_j = (\phi(A^t x_i)_j-\phi(A^t x'_i)_j)^2$ and denote $r = (r_1,\dots,r_k) \in \RR^k$.  Set $\bar{r} = \frac{1}{k}\sum_{j=1}^k r_j$. We have by definition that 
$\zeta_k(\phi(A^t x_i),\phi(A^t x'_i)) = \bar{r}$. 
To obtain an approximation in $\RR^l$, we will choose $l$ coordinates $r_j$ at random and consider the restriction to these coordinates. 

For $l>0$ let $\pi:\RR^k \rightarrow \RR^l$ be a random coordinate projection. That is, choose indices $i_1,\ldots,i_{l}$ independently at random from $\Set{1,\ldots,k}$ and set $\pi(v) = \Brack{v_{i_1},\ldots,v_{i_l}}$ for all $v\in \RR^k$. 
Denote by $\hat{r}$ the empirical average $\hat{r} = \overline{\pi r} = \frac{1}{l}\sum_{j=1}^l (\pi r)_j$.

Note that for every $j\leq k$, $\phi(\inner{A_{\cdot j}}{x_i}) \leq a b \norm{\phi}_{Lip} $ and thus  $\norm{r}_{\infty} \leq a^2 b^2 \norm{\phi}_{Lip}^2$.
Therefore by Bernstein's inequality, \cite{boucheron_concentration_2013}, 
\begin{equation}
    \Probsubidx{\pi}{\Abs{\hat{r} - \bar{r}}>t }\leq 2 exp\Brack{- l \cdot t^2 /  \Brack{a b\norm{\phi}_{Lip}}^4}.
\end{equation} 
Finally, recall that $r$ depends on $i$. Repeating this for every $i\leq n$ and using the union bound, 
\begin{equation} 
\label{eq:f_k_f_l_approx_prob}
\Probsubidx{\pi}{ \max_{i\leq n} \Abs{\hat{r} - \bar{r}}>t }\leq 2n exp\Brack{- l \cdot t^2 /  \Brack{a b\norm{\phi}_{Lip}}^4}.
\end{equation} 
Choosing $t_0 = \frac{\Brack{a b\norm{\phi}_{Lip}}^2 \sqrt{\log 4n}}{\sqrt{l}}$, the right handside of (\ref{eq:f_k_f_l_approx_prob}) is strictly smaller than 1. Therefore there is a choice $\pi$ such that for all $i\leq n$, 
$\Abs{\hat{r} - \bar{r}}\leq t_0$. 

Finally, given $A \in \RR^{d \times k}$ and $\pi = (i_1,\ldots,i_l)$, denote by $A\pi \in \RR^{d \times l}$ the restriction of A to $\pi$, $A\pi = (A_{\cdot i_1},\ldots,A_{\cdot i_l})$. Note that if $A\in \mc{G}_a^k$ then $A\pi \in \mc{G}_a^l$.  
If one denotes $f_A(i) = \zeta_k(\phi(A^tx_i),\phi(A^t x'_i))$, then we have shown that for the particular $\pi$ found above $\max_{i\leq n} \Abs{f_A(i) - f_{A\pi}(i)} \leq t_0$, thereby completing the proof. 
\end{proof}
\newcommand{\figheight}{3.35cm}
\begin{figure*}
\centering
\subcaptionbox{
\label{fig:mnist_losses}
Losses}{\includegraphics[width=0.33\textwidth,height=\figheight]{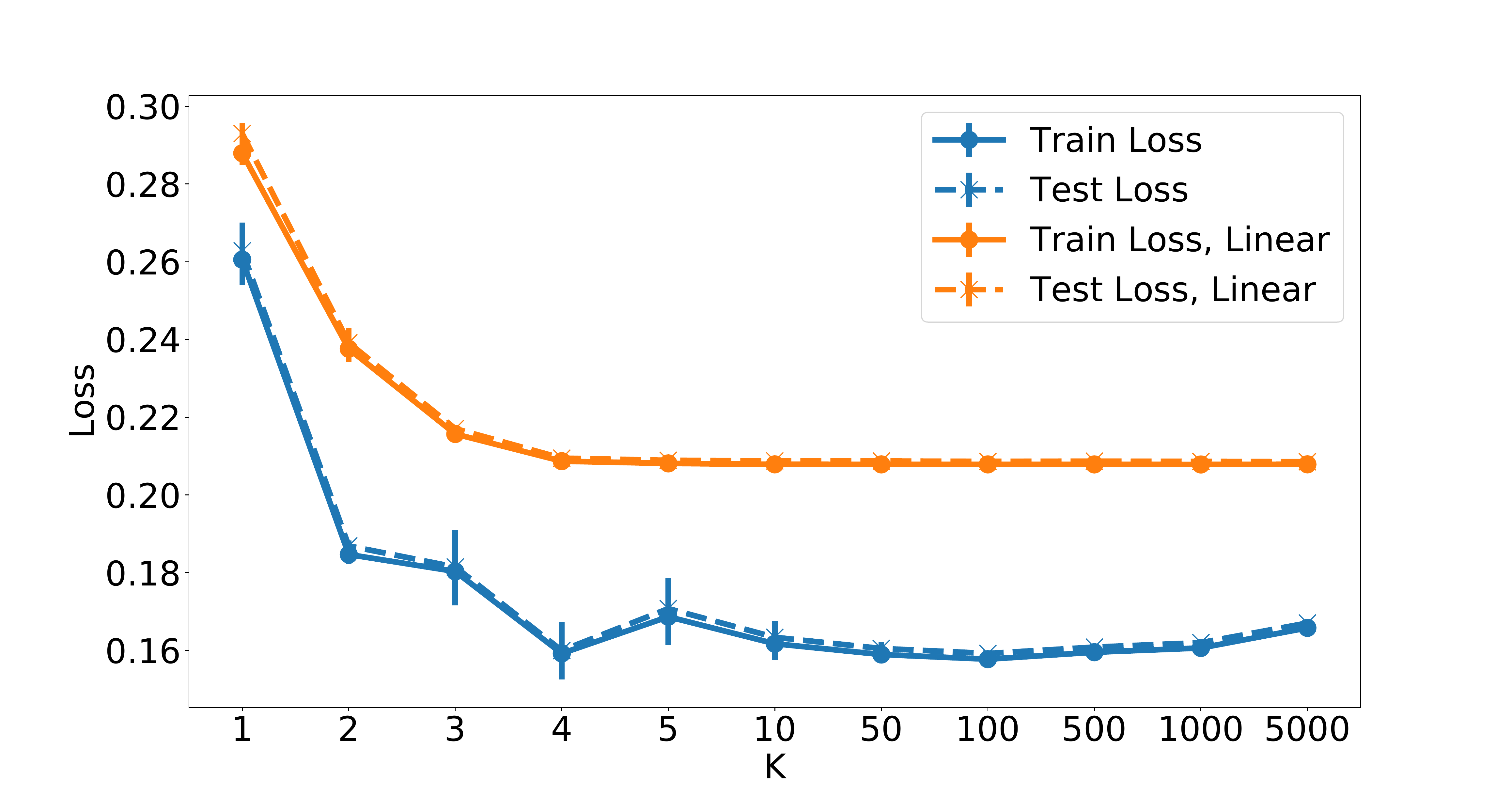}}%
\hfill 
\subcaptionbox{
\label{fig:mnist_weights}
$\norm{A}_{2,\infty}$, $\norm{A}_{2,1}/k$ and $\norm{A}_{op} / 
\sqrt{k}$ Weight Norms, Sigmoid Embedding.}{\includegraphics[width=0.33\textwidth,height=\figheight]{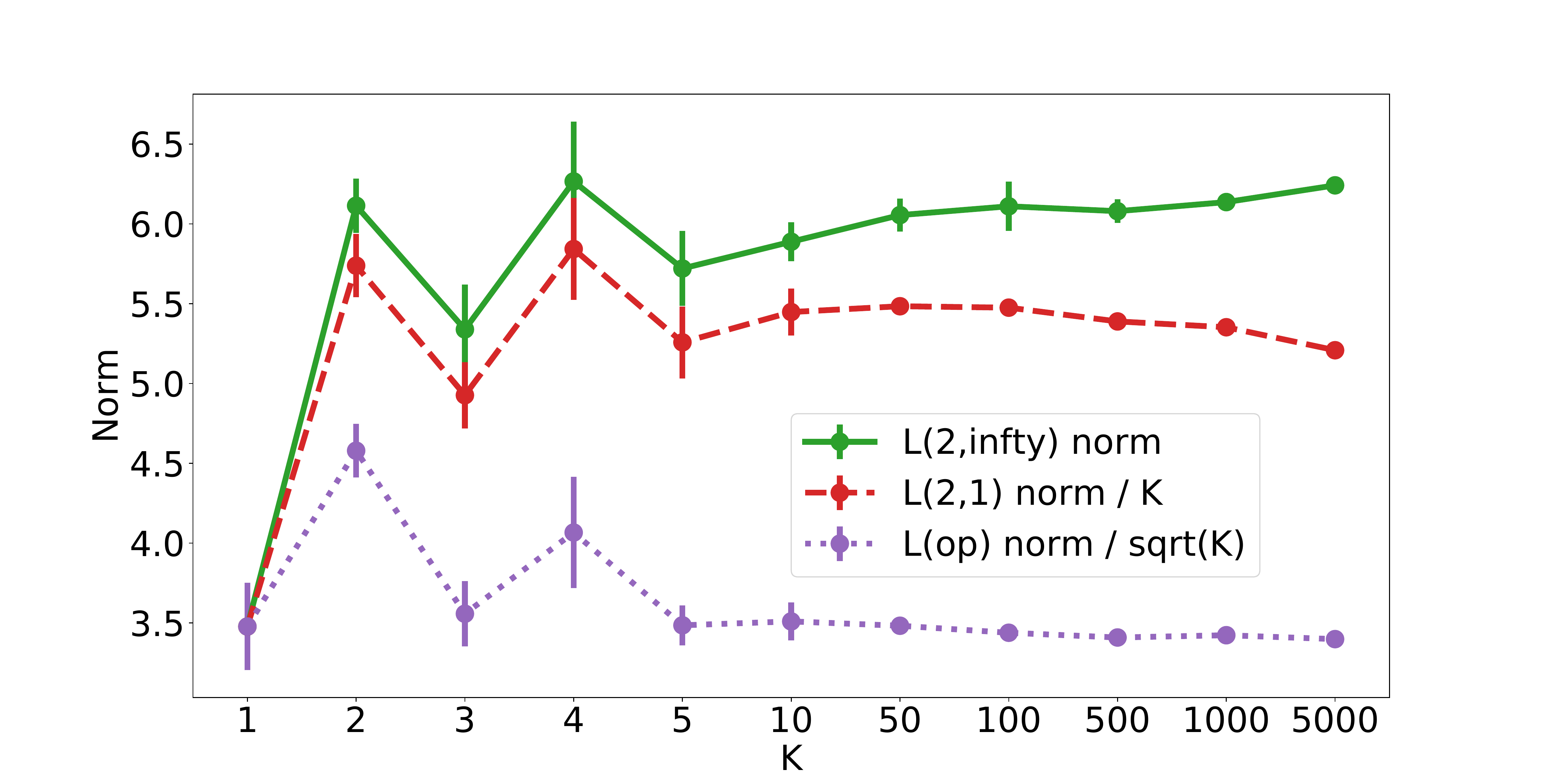}}%
\hfill 
\subcaptionbox{
\label{fig:mnist_posneg}
Same/Different Class Components of the Loss}{\includegraphics[width=0.33\textwidth,height=\figheight]{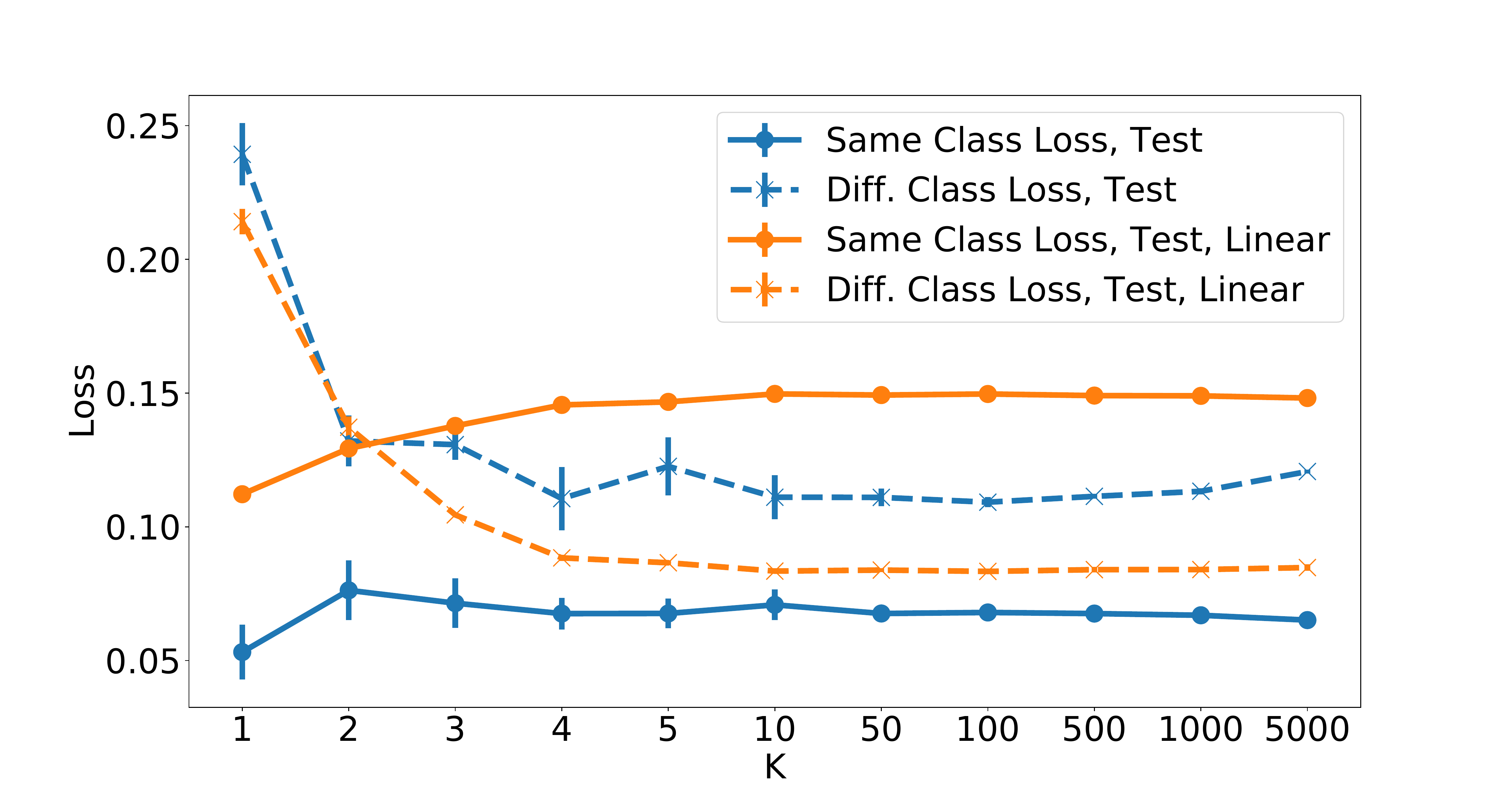}}%
\caption{\label{fig:mnist} MNIST Experiment}
\end{figure*}
To prove (\ref{eq:non_sparse_bound_n_quater}) from Lemma \ref{lem:approximation_f_k_fl}, it is suffices to note that for any $l<k$ one can write $\eRad{\mcF_{k}(\mc{G}^k_a)} \leq \eRad{\mcF_{l}(\mc{G}^l_a)} +\eRad{\Delta \mcF_{k,l}}$ where 
$\Delta \mcF_{k,l} = \Set{f - \hat{f} \setsep f\in \mcF}$. The first term can be bounded by Corollary \ref{cor:rad_non_sparse}, the second by Lemma \ref{lem:approximation_f_k_fl}, and an appropriate choice of $l$ yields (\ref{eq:non_sparse_bound_n_quater}). The full details are given in Supplementary Material Section \ref{sec:appendix_dim_red_proof_of_non_sparse_result}.

We conclude this Section with a brief discussion of the multi-layer networks and of the role of the normalization by $k$ in (\ref{eq:rho_def}). 
\begin{rem}
\label{rem:multi_layer_networks}
The special feature of metric learning compared to a regular classification problem is the particular structure of the loss. That is, that the loss depends on the features only through the metric $\rho$. 
In Theorem \ref{thm:rad_direct_bound}, however, we have only exploited the Lipschitzity of $\rho$. We therefore note that it is possible to obtain a multi-layer version of Theorem \ref{thm:rad_direct_bound} simply by replacing 
in our proof the single layer covering numbers bound
(\ref{eq:lem:A_21_regularized_bounds_statement}) with a multi-layer bound (\cite{bartlett2017spectrally}, Theorem 3.3).  

Moreover, since both our proofs of Theorem \ref{thm:rad_non_sparse_dim_free} eventually invoke Theorem \ref{thm:rad_direct_bound} for low $k$, it is clearly  possible to also extend Theorem \ref{thm:rad_non_sparse_dim_free} to the multi-layer case. In that case, the bound would depend on the $\norm{\cdot}_{op}$ and $\norm{\cdot}_{2,1}$ norms of all the layers except the last one, in a way similar to Theorem 1.1, \cite{bartlett2017spectrally}. The dependence on the last layer, however, would be only through $\norm{\cdot}_{2,\infty}$. In particular, we would still have $\sqrt{k}$ factor improvement over Theorem \ref{thm:rad_direct_bound}. On the other hand, it seems unlikely that the special properties of $\rho$ can be pushed further, to have only $\norm{\cdot}_{2,\infty}$ dependence also in lower layers.
\end{rem}

\begin{rem}
\label{rem:normalization_of_metric_by_k}
As noted earlier, the  normalization by $k$ in (\ref{eq:rho_def}) (or in the definition of $\zeta_k$) is crucial to the fact that (\ref{eq:linear_bound_non_sparse}) is dimension free. However, this is a natural normalization, ensuring that the values of $\rho_{k,f^*}$ are of the same magnitude (
$0 \leq \rho_{k,f^*} \leq 4 \norm{A}_{2,\infty}^2 b^2 $) independently of $k$, rather than explode with $k$. Note that if one normalized by anything of larger order than $k$, then the situation would have been degenerate, since $\rho_{k,f^*}$ values would simply be vanishing as $k$ grows. However, with normalization by $k$ this is not the case. The values of $\rho_{k,f^*}$ stay bounded but non-vanishing, while the number of parameters grows with $k$. Equivalently, while the magnitude of the values of $\rho_{k,f^*}$ stays bounded, 
the family of functions $\mcF_{k}(\mc{G}^k_a)$ becomes strictly larger as $k$ grows. Thus the fact that the Rademacher complexity (\ref{eq:linear_bound_non_sparse}) is bounded independently of $k$ is non-trivial. 
\end{rem}

\section{Experiments}
\label{sec:experiments}
In this Section we empirically make the following observations: (\textbf{a}) The bounded amplification regime introduced in this paper, where $\norm{A}_{2,\infty}$ is bounded while $\norm{A}_{2,1}$ and  $\norm{A}_{op}$ grow with $k$, occurs naturally in practice. For instance, we observe this on the MNIST data, trained without any regularization terms.
(\textbf{b}) The generalization gap, i.e. the difference between train and test loss, does not grow with $k$, for $k$ in the range from 1 to 5000, as suggested by the theory. In other words, we can increase $k$ arbitrarily, without risking overfitting. 
(\textbf{c}) The non-linear embeddings, even with a single layer, can perform better than the linear ones. Note that this point is not necessarily completely obvious. The linear embeddings are fairly powerful, and can, for instance, do things such as removing coordinates that are irrelevant to the label, thereby improving the relation between the metric and the label. Nevertheless, we find that on both our datasets, a sigmoid non-linearity improves the performance of the embedding. 

We consider two classical datasets, the MNIST dataset of handrwitten digits \cite{mnist_dataset},  and the 20newsgroups dataset \cite{newsgroups_dataset}, which consists of newsgroups emails, labeled according to the group. To make illustrations simpler, we restrict the 20newsgroups to the first 10 labels. 
While MNIST is generally nicely behaved, the 20newgroups is not. In particular, it consists of about $n=7500$ samples in dimension $d=15000 > n$, and this would require explicitly regularizing $\norm{A}_{2,\infty}$. 
In all cases, we consider a single layer fully connected embedding,  as in (\ref{eq:non_lin_emebdding_one_layer_def}), where the activation function $\phi$ is either $\phi(x) = x$ (linear case) or $\phi(x) = \sigma(x) = \frac{1}{1 + e^{-x}}$.

\begin{figure*}
\centering
\subcaptionbox{
\label{fig:ng_losses}
Losses}{\includegraphics[width=0.33\textwidth,height=\figheight]{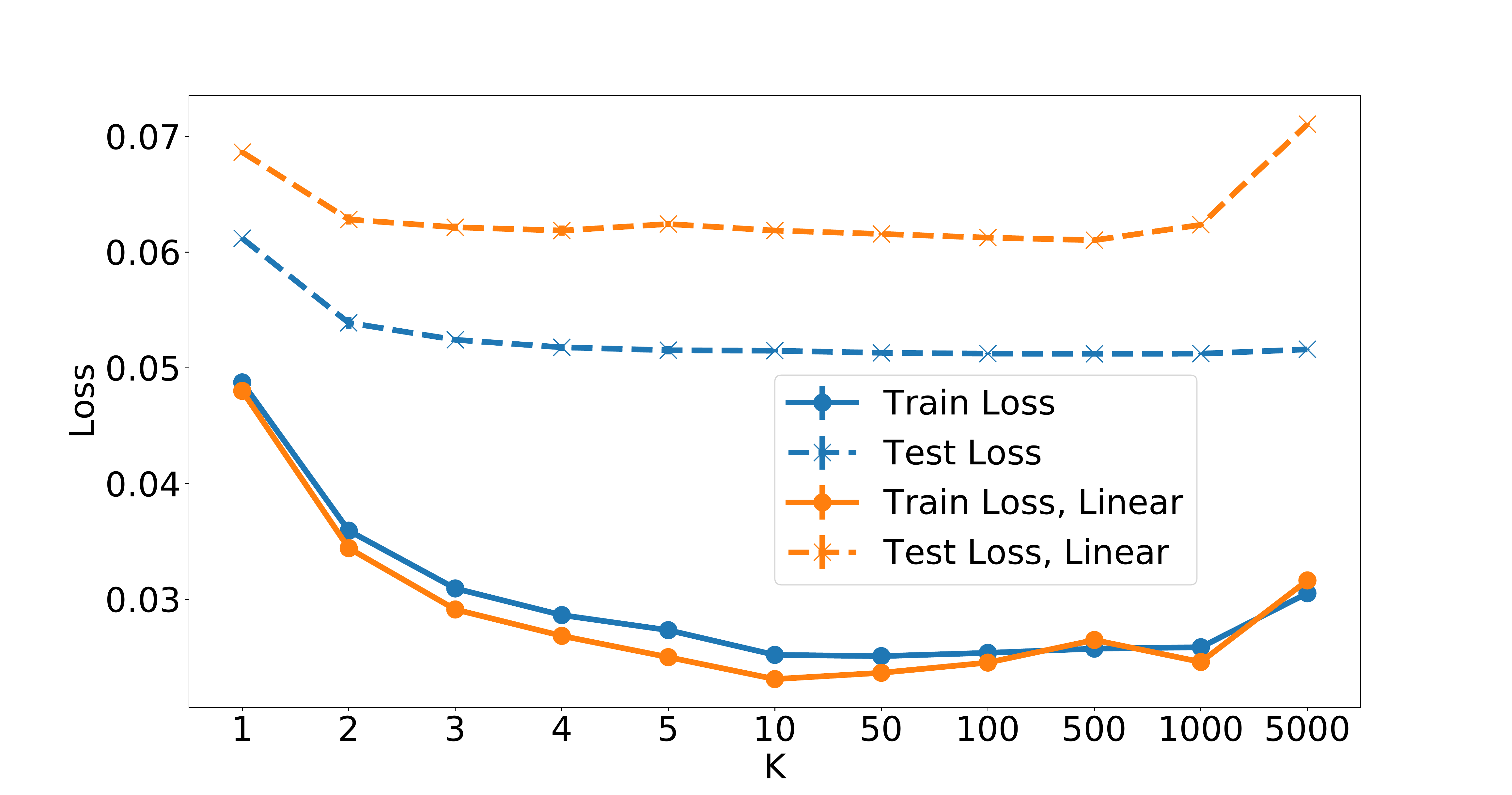}}%
\hfill 
\subcaptionbox{
\label{fig:ng_weights}
$\norm{A}_{2,\infty}$, $\norm{A}_{2,1}/k$ and $\norm{A}_{op} / 
\sqrt{k}$ Weight Norms, Sigmoid Embedding.}{\includegraphics[width=0.33\textwidth,height=\figheight]{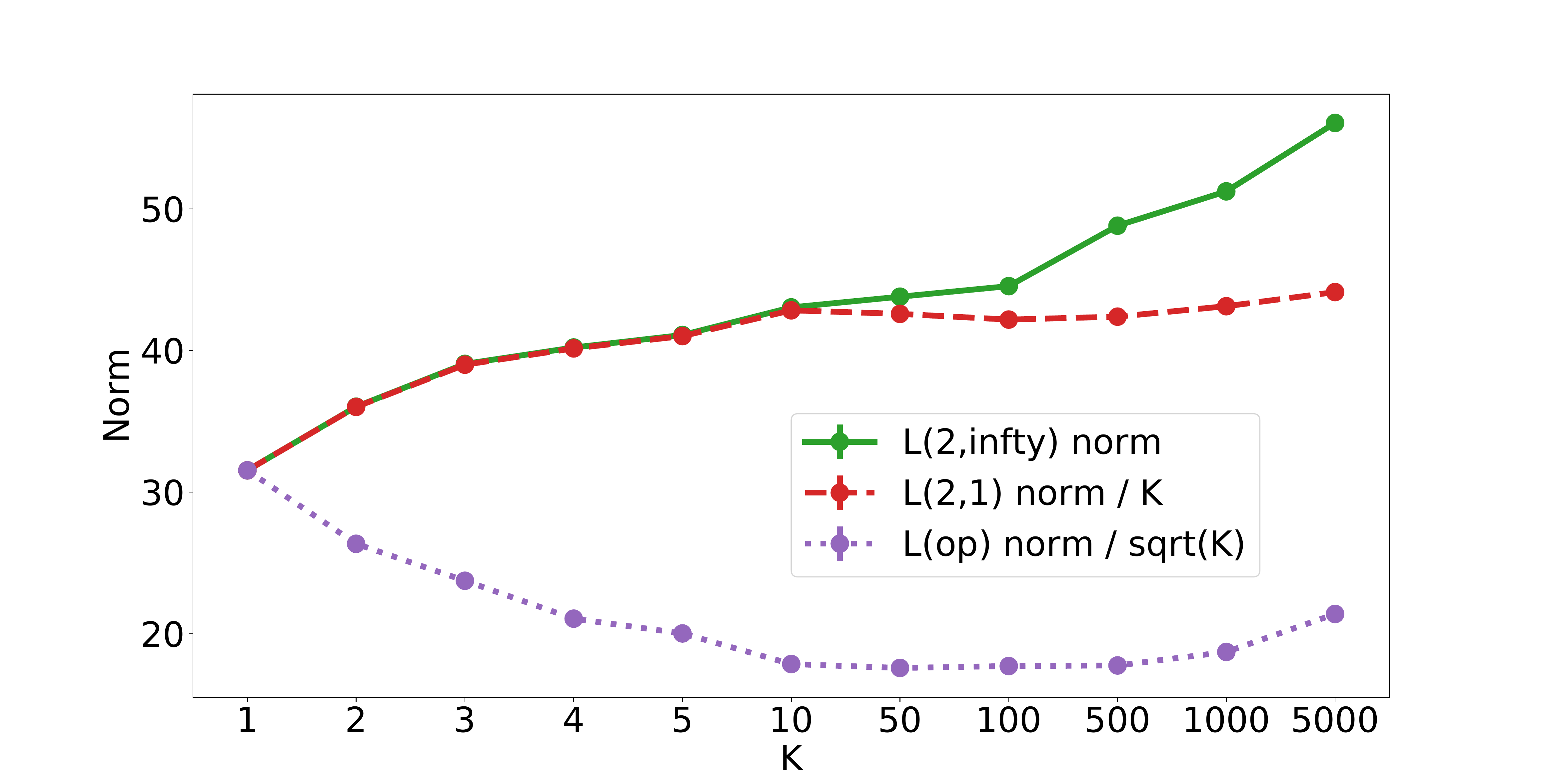}}%
\hfill 
\subcaptionbox{
\label{fig:ng_posneg}
Same/Different Class Components of the Loss}{\includegraphics[width=0.33\textwidth,height=\figheight]{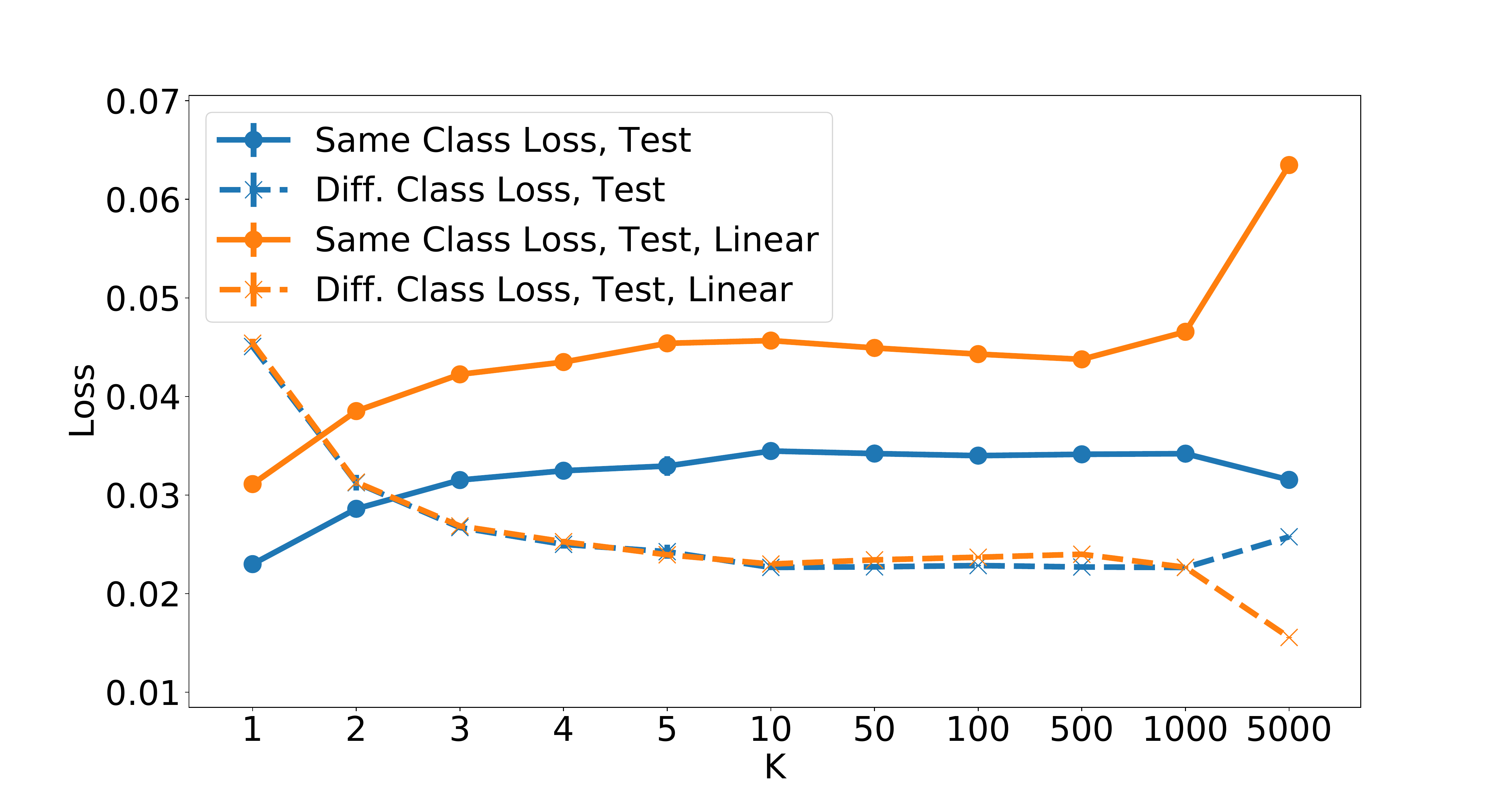}}%
\caption{\label{fig:ng} Newsgroups Experiment}
\end{figure*}

To perform the optimization, we sample feature/label pairs $(x,l)$, $(x',l')$ independently from the train set, and minimize the loss using SGD on batches of such pairs, until convergence. The loss that we use is 
\begin{flalign}
    \ell(x,x',l,l') =&
    9\cdot ReLU(\rho(x,x')) \cdot \Ind{l=l'} \nonumber \\ 
    &+ ReLU(D - \rho(x,x')) \cdot \Ind{l \neq l'}, \label{eq:empirically_used_loss}
\end{flalign}
where $\rho(x,x')$ is the distance after the embedding, given by (\ref{eq:rho_def}). This loss is a version of the loss $\ell_{S,D}^{\lambda}(\rho,y)$ introduced in Section \ref{sec:background_separate_section}, where the same class case is weighted by 9. Our datasets have 10 roughly balanced labels, and therefore the probability of the event $l=l'$ is about $1/10$. Thus the multiplier 9 helps in giving similar weights to the separation $l \neq l'$ and the compression $l = l'$ conditions, enforced by the respective terms in (\ref{eq:empirically_used_loss}). For train or test evaluation, we sample the pairs $(x,l)$, $(x',l')$ from the train or test set, respectively, and compute the mean value of $\ell(x,x',l,l')$ on these samples. 
For MNIST we use the upper threshold $D=0.5$, while for 20newsgroups we set $D=0.1$.

Each experiment was repeated 6 times, and every value in Figures \ref{fig:mnist}, \ref{fig:ng} is a mean over 6  outcome values. Every value also has an error bar which indicates the magnitude of the standard deviation around the mean. However, in most cases, these bars are so small compared to the magnitude of the mean, that they are not visible in the figures. Additional details on the experimental setting are given in Supplementary Material Section \ref{sec:supp_additional_experimental_details}. The full code that was used for the experiments is also provided as Supplementary Material. We now discuss each dataset separately. 
\subsection{MNIST}
For MNIST, the feature dimension is $d = 28\cdot 28 = 784$. As mentioned earlier, in this experiment we do not use any regularization terms on the embedding weight matrix $A$. 
The results are shown in Figure \ref{fig:mnist}.   
First, Figure \ref{fig:mnist_weights} shows the norms of $A$ after the convergence of the optimization,  for the sigmoid embeddings. The $\norm{A}_{2,\infty}$ is given by the green solid line, for values of $k$ ranging from 1 to 5000. The values of  $\norm{A}_{2,1}/k$ and $\norm{A}_{op}/\sqrt{k}$ are given by the dashed red and purple lines, respectively. The fact that all three lines remain of the same magnitude throughout the whole range of $k$ means that the standard non reguilarized SGD optimization works in this case in the bounded amplification regime. That is, the $\norm{A}_{2,\infty}$ norm remains roughly bounded while   $\norm{A}_{2,1}$ grows linearly with $k$ and $\norm{A}_{op}$ grows linearly with $\sqrt{k}$. The situation for the linear case  is similar (see Figure \ref{fig:mnist_weights_linear} in the Supplementary Material).

Next, we consider the losses on train and test sets, Figure \ref{fig:mnist_losses}. For the sigmoid embedding, the blue solid line is the loss evaluated on the train set, while the blue dashed line is the test loss. Observe that the lines are practically identical as $k$ grows. There is no deterioration of the generalization gap as $k$ grows. This may appear  somewhat counterintuitive at first. However, according to Figure \ref{fig:mnist_weights}, we are in the non-sparse, bounded $\norm{A}_{2,\infty}$ norm regime. Therefore our results, in particular Theorem \ref{thm:rad_non_sparse_dim_free}, suggest that indeed the generalization gap should be independent of $k$. Similar picture occurs for the linear case, as shown by the orange lines. Note that the non-linear loss is considerably lower than the linear one.

Finally, since the loss expression (\ref{eq:empirically_used_loss}) is rather involved, 
 it may be unclear from the raw values of the loss whether there is actually any good separation of the labels in the learned metric. In Figure \ref{fig:mnist_posneg} we show explicitly the same class and different class components of the loss. That is, for the sigmoid embedding,  the solid blue line in Figure \ref{fig:mnist_posneg} is the average of the quantity $ReLU(\rho(x,x'))$  over pairs of features $(x,x')$ on \emph{test set} for which $l=l'$ (that is, conditioned on the fact that labels are equal), while the dashed blue line is the average of $ReLU(0.5 - \rho(x,x'))$ on pairs with $l \neq l'$. Thus for instance, for $k=10$, at least on average, points with the same label are at distance $0.07$ while points with different labels are at distance $0.5 - 0.12 = 0.38$. 

\subsection{Newsgroups}
As mentioned earlier, the 20newsgrous dataset \cite{newsgroups_dataset} was restricted to the first 10 labels. Each email was represented as a bag-of-words vector 
over $d=15000$ most common tokens in the dataset. Each vector was then normalized to have Euclidean norm 1.  
This dataset contains 9640 samples, of which $20\%$ were taken at random as the test set, leaving $n=7700$ train samples. 
Since $d>n$, one can in principle severely overfit the data even with $k=1$. 
Thus one has to control the $\norm{A}_{2,\infty}$ norm, which was achieved by adding the term 
$\frac{0.1}{d} \norm{A}^2_{2,\infty}$ to the loss
(\ref{eq:empirically_used_loss}) in the sigmoid case, and 
$\frac{1.0}{d} \norm{A}^2_{2,\infty}$ in the linear case. The regularization values  $0.1$ and $1.0$ above where chosen such as to yield the best test performance at $k=10$ in their respective cases.

The experiment results are shown in Figure \ref{fig:ng}, with plots similar to those of the MNIST dataset. In particular, we see in Figure \ref{fig:ng_weights} that while we enforce the 
bounded $\norm{A}_{2,\infty}$ norm constraint with the regularization term, there was no sparsity, i.e.
$\norm{A}_{2,1}$ and $\norm{A}_{op}$ grew with $k$, as in the MNIST case. The generalization gap, Figure \ref{fig:ng_losses} was considerable, but it was \emph{the same gap} throughout the larger values of $k$ ($k\geq 5$). This can also be seen in Figure \ref{fig:ng_posneg} -- the values of the same class/different class expectations on the test set remain relatively constant over the range of $k$. Recall that the test set for $k=50$ was shown in Figure \ref{fig:newsgroup_tsne_b}, and corresponds to fairly good label separation, as can be expected from the values in Figure  \ref{fig:ng_posneg}.

\section{Conclusion}
\label{sec:conclusion}
In this work we have provided the first uniform 
generalization guarantees for metric learning with a non-linear embedding. In particular, we have shown that dimension free generalization is possible in the non-sparse regime. 

We believe that further study of the non sparse regime, and in particular understanding better in which other situations the non sparse regime occurs, and what are the associated bounds, would significantly advance our understanding overparametrized systems. In addition, observe that almost \emph{any} family of sets $\mc{U}_l \subset \RR^n$ , $l\geq 1$, (given, say, trivial compactness constraints) has the $l$-approximability property, as introduced in Section \ref{sec:non_sparse_proof}, for \emph{some}, necessarily vanishing, rate. We believe that the study of machine learning problems through the lens of such rates of their induced restrictions $\mc{U}_l$ is a very  promising new research direction. 

\bibliographystyle{apalike}
\bibliography{partitions_metrics.bib}

\appendix

\onecolumn
\begin{center} 
{\Large \textbf{Dimension Free Generalization Bounds for Non Linear Metric Learning }}
\newline
\newline
{\Large \textbf{Supplementary Material}}
\newline
\end{center}

\section{Proof of Theorem \ref{thm:rad_direct_bound}}
\label{sec:proof_of_thm_rad_sparse}

We will require the following bound on the Lipschitz constant of $\zeta_k$.
\begin{lem}
\label{lem:lip_of_kappa}
For $x,x',y,y' \in \RR^k$  set 
 $\gamma = \max \Set {\norm{x}_2,\norm{x'}_2,\norm{y}_2,\norm{y'}_2}$. Then
\begin{equation}
\label{eq:lip_lem_l2_bound}
\Abs{\zeta_k(x,x')-\zeta_k(y,y')} \leq \frac{8\gamma}{k} \Brack{\norm{x-y}_2 + \norm{x'-y'}_2}.
\end{equation}
\end{lem}
\begin{proof}
Set $\Delta = x-y, \Delta' = x' - y'$. 
Clearly $\norm{x-x'}_2\leq 2\gamma$, and $\norm{\Delta-\Delta'}_2\leq 4\gamma$. We have
\begin{flalign*}
&\Abs{\zeta_k(x,x')-\zeta_k(y,y')} \\
&= 
 \frac{1}{k}\Abs{\norm{x-x'}_2^2 - \norm{x-x'-\Delta + \Delta'}_2^2 } \\ 
 &= \frac{1}{k}\Abs{ -2\inner{x-x'}{\Delta - \Delta'} - 
 \norm{\Delta - \Delta'}_2^2}  \\ 
&= \frac{1}{k}\Abs{ -2\inner{x-x'}{\Delta - \Delta'} - 
 \inner{\Delta - \Delta'}{\Delta - \Delta'}}  \\ 
&\leq \frac{1}{k}\Brack{ 4\gamma\Brack{\norm{\Delta}_2 +\norm{\Delta'}_2} +  
4\gamma\Brack{\norm{\Delta}_2 +\norm{\Delta'}_2}} .
\end{flalign*}
\end{proof}

\begin{proof}[Proof of Theorem \ref{thm:rad_direct_bound}]
Denote throughout of the proof $\mcF := \mcF_{k}(\mc{J}^k_{a,a'})$. Our plan is to bound the covering numbers $\log \mc{N}(\mcF,\eps)$ and then to use the Dudley entropy bound for the Rademacher complexity. 
Fix $\eps>0$. Let $\mc{M}$ and $\mc{M}'$ be $\eps$-covers for the sets $X\mc{J}_a^k$ and $X'\mc{J}_a^k$, such  that $\log \Abs{\mc{M}}$ and $\log \Abs{\mc{M}'}$ are at most $\frac{a^2b^2n}{\eps^2}\log(2dk)$, as guaranteed by Lemma \ref{lem:A_21_regularized_bounds}. Then the set 
$\mc{M}\times\mc{M}'$ is an $\sqrt{2}\eps$-cover for the set $\mc{Z} = \Set{ (XA,X'A) \setsep A \in \mc{J}_{a,a'}^k } \subset \RR^{n\times 2k}$. Note that the elements of 
$\mc{M}\times\mc{M}'$ are not necessarily members of $\mc{Z}$. However, by taking projections onto $\mc{Z}$, we can find an $2\sqrt{2}\eps$-cover of size at most $\Abs{\mc{M}}\times \Abs{\mc{M}'}$ with elements inside $\mc{Z}$.  Denote such a cover by $\mc{P}$.

Let $\phi \mc{Z}$ denote the image of $\mc{Z}$ under $\phi$, acting coordinatewise, and denote by $\hat{\zeta_k}: \RR^{n\times 2k} \rightarrow \RR^n$ the map that applies $\zeta_k$ on the rows. In particular, for every $A$ and every coordinate $i \leq n$, we have 
\begin{equation}
    \hat{\zeta_k}\Brack{\phi\Brack{\Brack{XA,X'A}}}_i = \zeta_k\Brack{\phi(A^tx_i), \phi(A^tx'_i)}.
\end{equation}
Equivalently, by definition, we have that 
$\mcF = \hat{\zeta_k} \Brack{ \phi \Brack{\mc{Z}}} $.

Finally, since $\phi(0)=0$, for any $x_i$ and $A\in \mc{J}_{a,a'}^k$ we have
\begin{equation}
\label{eq:phi_A_x_i_bound}
\norm{\phi(A^tx_i)}_2 \leq b a' \norm{\phi}_{Lip}.
\end{equation}
It therefore follows by Lemma \ref{lem:lip_of_kappa} that 
$\hat{\zeta_k}$ is $\frac{\sqrt{2} 8 b a' \norm{\phi}_{Lip}}{k}$-Lipschitz as a function from $\phi(\mc{Z})$ to $\RR^n$.  Consequently, a $2\sqrt{2} \eps$-cover $\mc{P}$ of $\mc{Z}$ is mapped by $\hat{\zeta_k}\circ \phi$ into an $2\sqrt{2} \cdot \frac{\sqrt{2} 8 b a' \norm{\phi}_{Lip}}{k} \cdot \norm{\phi}_{Lip} \cdot \eps$-cover of $\mc{F}$.

Combining the above statements, we have shown that there is an absolute constant $c>0$ such that for any $\eps>0$, \begin{equation}
\label{eq:cover_bound_full_direct_thm}
\log \mc{N}( \mcF,\eps )   \leq 
c \frac{a^2 a'^2 b^4 n \norm{\phi}_{Lip}^4}{k^2 \eps^2} \log(2dk). 
\end{equation}

Set $D = diam(\mcF) = \sup_{u,v\in \mcF} \norm{u-v}_2$. Using (\ref{eq:phi_A_x_i_bound}) again, we have $D\leq \frac{4 b^2 a'^2 \norm{\phi}_{Lip}^2}{k} \cdot \sqrt{n}$. 
Next, by Dudley's entropy bound, \cite{vershynin_high_dimensional_2018}, and using (\ref{eq:cover_bound_full_direct_thm}),
for any $\delta>0$,
\begin{flalign*}
\eRad{\mcF} &\leq 
\frac{4\delta}{\sqrt{n}} + \frac{12}{n}\int_{\delta}^{D/2} 
\sqrt{\log \mc{N}( \mcF,\eps )} d\eps \\ 
&\leq 
\frac{4\delta}{\sqrt{n}} + 
\frac{12\sqrt{c}}{\sqrt{n}} \frac{a a' b^2  \norm{\phi}^2_{Lip}}{k} \log\Brack{D / \delta}.
\end{flalign*}
Taking $\delta = \frac{1}{\sqrt{n}}$ and using the above bound on $D$ completes the proof. 
\end{proof}

\section{Proof of Corollary \ref{cor:rad_non_sparse}}
\label{sec:cor_rad_non_sparse_proof}

\begin{proof}
For any $A \in \RR^{d \times k}$, we have  
$\norm{A}_{2,1} \leq k \norm{A}_{2,\infty}$ and 
$\norm{A}_{op}\leq \sqrt{k}\norm{A}_{2,\infty}$. The first inequality follows immediately from the definitions. For the second inequality observe that
\begin{flalign*}
    \norm{A}_{op} = \norm{A^t}_{op} &= \sup_{\norm{v}_2 = 1} \norm{A^tv} = 
    \sup_{\norm{v}_2 = 1} 
    \Brack{\sum_{i=1}^k \inner{A_{\cdot i}}{v}^2}^{\half} \\&\leq
    \Brack{k \max_{i\leq k} \norm{A_{\cdot i}}_2^2}^{\half} = \sqrt{k} \norm{A}_{2,\infty}.
\end{flalign*}
The two inequalities above imply that 
$\mc{G}_a^k \subset \mc{J}_{ka,\sqrt{k}a}^k$ and thus the Corollary follows from Theorem \ref{thm:rad_direct_bound}.
\end{proof}

\section{First Proof of Theorem \ref{thm:rad_non_sparse_dim_free}}
\label{sec:appendix_additivity_proof_of_non_sparse}
\begin{proof}[Proof 1 of Theorem \ref{thm:rad_non_sparse_dim_free}:]
For any set of families of functions $\mc{H}_i$, $i\leq m$, define the sum 
$\oplus_{i=1}^m \mc{H}_i$ by 
$\oplus_{i=1}^m \mc{H}_i =  \Set{ g = \sum_{i=1}^m h_i \setsep h_i \in \mc{H}_i}$. For a scalar $\alpha$ set also $\alpha \mc{H} = \Set{\alpha h \setsep h \in \mc{H}}$. We observe the following:
\begin{equation}
\label{eq:G_k_as_set_sum}
\mc{F}_k(\mc{G}^k_a) = \frac{1}{k} \Brack{\oplus_{i=1}^k \mcF_1(\mc{G}^1_a)},
\end{equation} 
where the expression on the right handside is the sum of $\mc{F}_1(\mc{G}^1_a)$ taken $k$ times, normalized by $\frac{1}{k}$.
This equality is a consequence of the definition of $\zeta_k$ as an average, and of the fact that $\mc{G}^k_a$ is a product set of its individual columns. More explicitly, given $A\in \mc{G}_a^k$, the vector in $\mc{F}_k(\mc{G}^k_a)$ corresponding to $A$ is, by definition, $\Brack{\zeta_k(\phi(A^tx_i), \phi(A^tx'_i)
    }_{i=1}^n \in \RR^n$.  For a fixed coordinate $i\leq n$ we 
have
\begin{flalign}
    \zeta_k(\phi(A^tx_i), \phi(A^tx'_i))
     &= \frac{1}{k} \sum_{j=1}^k 
    \Brack{\phi(  (A^tx_i)_j ) - \phi(  (A^tx'_i)_j ) }^2 \nonumber \\ 
    &=\frac{1}{k} \sum_{j=1}^k \Brack{
    \phi(  A_{\cdot j}^t x_i ) - \phi(  A_{\cdot j}^t x'_i )}^2, \label{eq:rad_proof_direct_sum_explained}
\end{flalign}
where $(A^tx_i)_j$ is the $j$-th coordinate of $(A^tx_i)$ and 
$A_{\cdot j}^t x_i$ is the product of $A$ restricted to its $j$-th column with $x_i$. Since each term in the sum (\ref{eq:rad_proof_direct_sum_explained}) is in $\mc{G}_a^1$, the claim (\ref{eq:G_k_as_set_sum}) follows.

Next, one can readily verify that for any set of families $\mc{H}_i$, we have $\eRad{\oplus_m \mc{H}_m} = \sum_{i=1}^m \eRad{\mc{H}_i} $, and $\eRad{\alpha \mc{H}} = \alpha \eRad{\mc{H}}$ for $\alpha \geq 0$ (see \cite{mendelson_bartlett_rad}, \cite{mohri_book_2018} ). Therefore, using (\ref{eq:G_k_as_set_sum}) we have 
\begin{equation}
\eRad{\mc{F}_k(\mc{G}^k_a)} = \eRad{\mc{F}_1(\mc{G}^1_a)}.
\end{equation}
The statement now follows from Corollary \ref{cor:rad_non_sparse} applied with $k=1$.
\end{proof}

\section{Second Proof of Theorem \ref{thm:rad_non_sparse_dim_free}}
\label{sec:appendix_dim_red_proof_of_non_sparse_result}

\begin{proof}[Proof 2 of Theorem \ref{thm:rad_non_sparse_dim_free}, Dimension Reduction:]
Fix $l<k$, and set $t_0 = \frac{\Brack{a b\norm{\phi}_{Lip}}^2 \sqrt{\log 4n}}{\sqrt{l}}$. By Lemma \ref{lem:approximation_f_k_fl}, for every $f \in \mcF_k$ one can write $f = \hat{f} + (f-\hat{f})$ such that $\hat{f}\in \mcF_l$ and $\norm{f-\hat{f}}_{\infty} \leq t_0$. It follows that 
\begin{flalign*}
    &\eRad{\mcF_k} = \frac{1}{n}\Expsubidx{\eps}{sup_{f\in \mcF_k} \sum_{i=1}^n \eps_i f_i  }  \\ 
    &\leq  \frac{1}{n}\Expsubidx{\eps}{sup_{f\in \mcF_l} \sum_{i=1}^n \eps_i f_i  } 
    + \frac{1}{n}\Expsubidx{\eps}{sup_{\norm{f}_{\infty}\leq t_0} \sum_{i=1}^n \eps_i f_i  } \\ 
    &\leq  \eRad{\mcF_l}   + t_0 \\ 
    &\leq 
\Brack{\frac{1}{n} + 
 c \cdot \frac{a^2 b^2 \sqrt{l} \norm{\phi}^2_{Lip}}{\sqrt{n}} \log \Brack{n a^2 \sqrt{l} \norm{\phi}_{Lip}}} 
 \\ &\spaceo\spaceo\spaceo
 +\frac{\Brack{a b\norm{\phi}_{Lip}}^2 \sqrt{\log 4n}}{\sqrt{l}},
\end{flalign*}
where the last inequality follows by Corollary \ref{cor:rad_non_sparse}. Minimizing the above expression in $l$ yields (\ref{eq:non_sparse_bound_n_quater}).
\end{proof}

\section{Additional Details on Experiments}
\label{sec:supp_additional_experimental_details}

\begin{figure*}
\centering
\subcaptionbox{
\label{fig:mnist_weights_linear}
MNIST Experiment. $\norm{A}_{2,\infty}$, $\norm{A}_{2,1}/k$ and $\norm{A}_{op} / 
\sqrt{k}$ Weight Norms, Linear Embedding.}{\includegraphics[width=0.33\textwidth,height=3.5cm]{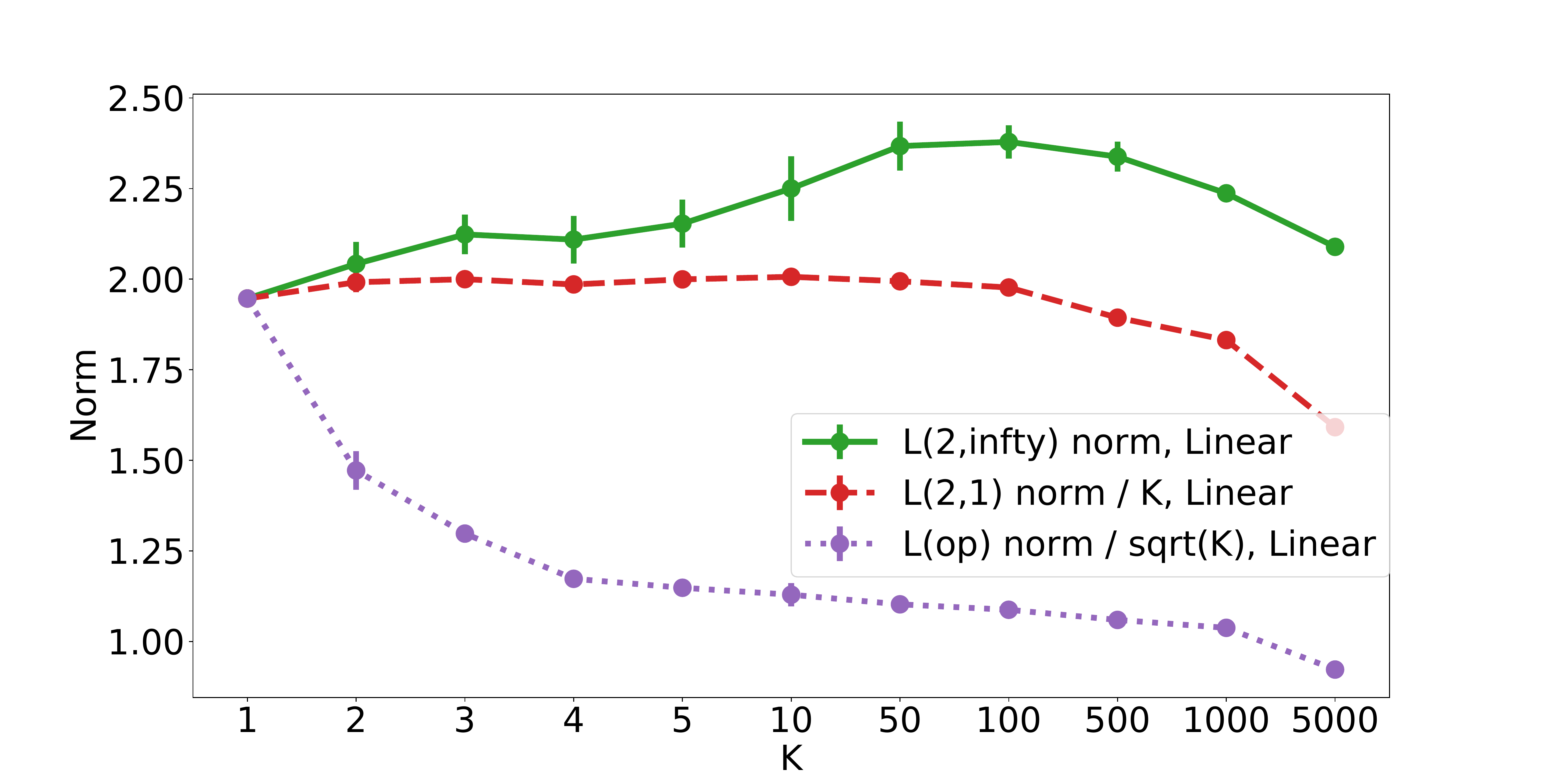}}%
\hfill 
\subcaptionbox{
\label{fig:ng_weights_linear}
Newsgroups Experiment. $\norm{A}_{2,\infty}$, $\norm{A}_{2,1}/k$ and $\norm{A}_{op} / 
\sqrt{k}$ Weight Norms, Linear Embedding.}{\includegraphics[width=0.33\textwidth,height=3.5cm]{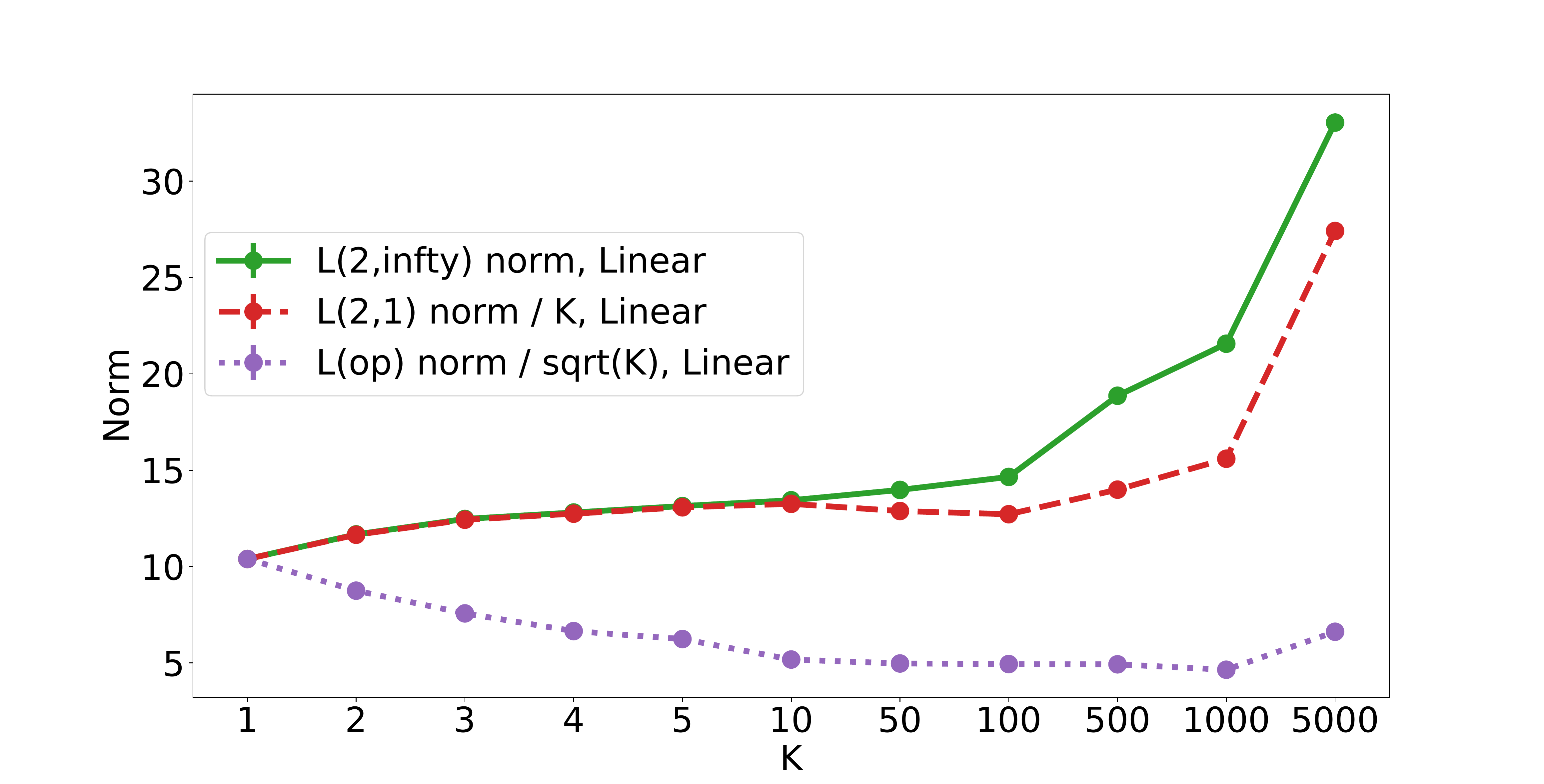}}%
\end{figure*}

\begin{itemize}
    \item 
The full code of the experiments is provided as a Supplementary Material. The code is based on the Tensorflow framework. Details on the invocation can be found in the README.md file. 
    
    \item The precise process of sampling the pairs $(x,x')$ was as follows: 
    We fix two batch sizes $Bx,By$, typically $Bx=By=500$.
    We take a batch $S_1$ of size $Bx$ and another, independent batch, $S_2$, of size $B_y$. Then the set $S$ of all feature pairs $(x,x')$ such that $x \in S_1$ and $x' \in S_2$ is created, with corresponding labels $y$ as described in Section \ref{sec:intro}. This $S$ is fed to the optimizer as a single batch. 
    
    \item An \emph{epoch} is when first batch iterator, $S_1$, finishes a single iteration over the whole data. The data is permuted at the end of the epoch. 
    
    \item All experiments were run for 180 epochs. Convergence typically occurred much earlier, around 30 to 50 epochs, depending on the value of $k$.
    
    \item All experiments were run on a GTX 1080 GPU, with Tensorflow 2.0. 
    A single run of the longest experiment, MNIST with $k=5000$, took under 150 minutes to complete. All experiments combined, with repetitions, finish running in under 48 hours. 
\end{itemize}

\end{document}